\documentclass{ecai}
\usepackage{latexsym}

\usepackage{times}
\usepackage{soul}
\usepackage{bm}
\usepackage{url}
\usepackage{color}
\usepackage{multirow}
\usepackage[hidelinks]{hyperref}
\usepackage[utf8]{inputenc}
\usepackage[small]{caption}
\usepackage{graphicx}
\usepackage{subfigure}
\usepackage{amsmath}
\usepackage{amsthm}
\usepackage{amssymb,amsfonts}
\usepackage{booktabs}
\usepackage{algorithm}
\usepackage{algorithmic}
\usepackage{diagbox}
\usepackage[switch]{lineno}
\usepackage{thm-restate} 
\usepackage{latexsym}
\usepackage{multirow} 
\usepackage{url}
\usepackage{float}
\usepackage{subfigure}
\usepackage{bm}
\usepackage{amsmath}
\usepackage{amsfonts}
\usepackage{nicefrac}     
\usepackage{microtype}      
\usepackage{lipsum}
\usepackage{color}
\usepackage{booktabs}
\usepackage{epsfig}
\usepackage{epstopdf}
\usepackage{pifont}
\usepackage{bbding}

\ecaisubmission   

\newcommand{\ligry}[1]{{\tiny{#1}}}

\newcommand{\fed}{FedCOME}

\newtheorem{definition}{Definition}

\newtheorem{thm}{Theorem}
\newtheorem{lemma}{Lemma}
\newtheorem{corollary}{Corollary}

\begin{document}

\begin{frontmatter}
\title{Federated Learning via Consensus Mechanism on Heterogeneous Data: A New Perspective on Convergence}

\author[\dag]{\fnms{Shu}~\snm{Zheng}}
\author[\dag]{\fnms{Tiandi}~\snm{Ye}}
\author[\(*\)]{\fnms{Xiang}~\snm{Li}
\footnote{
Corresponding author. Email: xiangli@ecnu.edu.cn. \\
$^\dagger$~ Authors contributed equally. Order chosen alphabetically.}
}
\author[]{\fnms{Ming}~\snm{Gao}}
\address[]{East China Normal University}

\begin{abstract}
Federated learning (FL) 
on heterogeneous data (non-IID data)
has recently received great attention.
Most existing methods focus on studying the convergence guarantees for the global objective.
While
these methods can guarantee the decrease of the global objective in each communication round, they fail to ensure risk decrease for each client. 
In this paper,
to address the problem,
we propose \fed,
which introduces a consensus mechanism to enforce decreased risk for each client after each training round.
In particular,
we allow a slight adjustment to a client’s gradient on the
server side, which generates an acute angle between the corrected gradient and the original ones of other clients.
We theoretically show that the consensus mechanism can guarantee the convergence of the global objective.
To generalize the consensus mechanism to the partial participation FL scenario,
we devise a novel client sampling strategy to select the most representative clients for the global data distribution.
Training on these selected clients with the consensus mechanism could empirically lead to risk decrease for clients that are not selected.
Finally, we conduct extensive experiments on four benchmark datasets to show the superiority of \fed\ against other state-of-the-art methods in terms of effectiveness, efficiency and fairness.
For reproducibility,
we make our source code publicly available at: \url{https://github.com/fedcome/fedcome}.
\end{abstract}
\end{frontmatter}

\section{Introduction}
Recent years have witnessed the great success of deep learning in a wide range of fields, such as computer vision~\cite{zhao2003face,gatys2016image}, natural language processing~\cite{devlin2018bert,radford2019language} and 
speech recognition~\cite{reddy1976speech,povey2011kaldi}. 
Typically, deep learning relies on a tremendous amount of high-quality data, 
which is usually 
scattered over thousands or millions of edge devices, such as mobile phones and IoT sensors. 
With the increasing awareness of data privacy protection and the data regulations coming into force, it is illegal to collect clients' data and perform model training in a centralized fashion. 
To address the problem, 
federated learning (FL) has emerged as a promising and powerful approach for model training from decentralized data without compromising clients' data privacy~\cite{kairouz2021advances,mcmahan2017communication,chen2021fedmatch,lin2020meta,muhammad2020fedfast,Chen2020FedHealthAF}.

Specifically,
FL is a special distributed machine learning paradigm, 
with the goal of minimizing the average risk over the population,
which proceeds in communication rounds between server and clients until convergence or reaching some termination criteria. 
In each round, the server first randomly selects a subset of clients. 
Then the selected clients download the latest model from the server and optimize it based on their private datasets. 
After that,
clients' local gradients (or model updates) are sent to the server after local training finishes and aggregated by the server to update the global model. 
Recently, 
FL on heterogeneous data (i.e., non-identically distributed data) across all the clients
has received great attention
\cite{zhao2018federated,li2019convergence,li2020federated,li2019feddane,karimireddy2020scaffold,wang2020tackling,acar2020federated}. 
With the assumption of bounded data heterogeneity, 
most existing works~\cite{li2020federated,li2019feddane} focus on theoretically proving the convergence guarantees for the global objective. 
While 
these works
can guarantee the decrease of the global objective in each communication round,
they fail to ensure risk decrease for each client.
Intuitively,
%
%
{a decrease in the risk of each client could further decrease the average risk over the whole population and gradually moves the global objective function to converge as training proceeds. 
} 
With this perspective, in this paper,
we study the convergence of the global objective function by 
proposing a \emph{consensus mechanism} 
that ensures the decreased risk for all the clients in each iteration under the full client participation assumption. 
In particular, 
we allow a slight adjustment to a client's gradient on the server-side, 
which generates an acute angle between the corrected gradient and the original ones of other clients. 
The consensus mechanism is theoretically proved to serve as a sufficient condition for the decreases of local risks for all the clients and further the decrease of the global objective over the whole population.

The consensus mechanism can only guarantee decreased risks for clients used in training. 
However, 
in the large-scale cross-device FL scenario, 
it is usual that
only a subset of clients are available.
To generalize the consensus mechanism to the partial client participation scenario,
we design a sampling strategy 
to select the most representative clients
for the global population.
Specifically, 
we first quantify the similarity between clients by their gradients
and then maintain a similarity table $S$, 
in which the $(i,j)$-th entry 
represents the similarity between the $i$-th and $j$-th clients. 
The sampling objective is to select a subset of clients $\mathcal{P}$ according to table $S$, where the sum of similarities between 
any two clients in $\mathcal{P}$ is minimized. 
After that,
the consensus mechanism can be applied on these selected clients to guarantee 
their decreased risks,
which
can also 
lead to empirical decreases of risks for clients that are not selected.

Finally,
based on the consensus mechanism and the sampling strategy, 
we develop 
a novel and easy-to-implement
\textbf{Fed}erated learning 
framework
via 
\textbf{CO}nsensus \textbf{ME}chanism, 
namely, 
\fed.
Different from most existing methods, 
\fed~ 
studies
the convergence of global 
objective in FL on heterogeneous data
by 
guaranteeing the decrease of each client's risk in each communication round,
which further leads to the decrease of the global risk.
In summary,
this paper
aims to 
shed light on the new direction for studying the data heterogeneity problem in FL
and provide insightful reference for future research on the topic.
Our main contributions in this paper are summarized as follows:
    
    \noindent{\small$\bullet$}
    We propose a novel FL framework \fed, which provides a new perspective for studying the convergence of the global objective in FL on heterogeneous data.
    
    \noindent{\small$\bullet$}
    We introduce a consensus mechanism to ensure decreased risk for all the clients in each communication round. 
    We theoretically prove that the consensus mechanism can guarantee the convergence of the global objective.
 
 \noindent{\small$\bullet$}
 We design an effective client sampling strategy to select most representative clients for the global data distribution, which can be used to generalize the consensus mechanism to the partial participation scenario. 
 
 \noindent{\small$\bullet$}
 We conduct
	extensive experiments on four benchmark FL datasets to demonstrate the superiority of \fed\ in terms of effectiveness, efficiency and fairness.	

\noindent

\section{Related Work}\label{sec:related-works}

\subsection{Data Heterogeneity} 
In real-world scenarios, 
data is generally heterogeneous 
across clients, 
which poses great challenges to FL.
To address the problem,
existing works can be mainly categorized into two types. 

To alleviate the discrepancy
of data distributions
among clients,
the first line of approaches 
regularize local model training
by adding various constraints 
to clients' optimization objectives~\cite{li2020federated,karimireddy2020scaffold,acar2020federated,chen2021bridging,pathak2020fedsplit}.
For example,
{FedProx}~\cite{li2020federated} adds a proximal term to regularize the distance between the global model and the local client model in parameter space. 
In~\cite{karimireddy2020scaffold},
the client drift problem is introduced and
SCAFFOLD is proposed to address the problem by correcting local updates with control variates. 
FedDC~\cite{gao2022feddc} utilizes an auxiliary local drift
variable to track and bridge the gap between the parameters of the local and the global model.
Further, 
FedDyn~\cite{acar2020federated} points out the misalignment between the stationary points of the client models and that of the server model, and 
aligns them through dynamically updating the local regularizer. 
In a special non-IID setting where clients have imbalanced class distributions, FedRod~\cite{chen2021bridging}  leverages balanced risk to align clients' objectives
and improve the model performance.

Another line of methods 
handle data heterogeneity 
with sophisticated model aggregation on the server~\cite{hsu2019measuring,wang2020tackling,wang2019federated,zhu2021data,lin2020ensemble,wang2021federated}. 
The most naive aggregation scheme is the 
coordinate-based parameter averaging adopted by {FedAvg}~\cite{mcmahan2017communication}. 
Further,
Wang~et~al.~\cite{wang2019federated}
point out the misalignment issue between clients' models due to the permutation invariant property of neural networks,
and they offer a matched averaging scheme to mitigate the issue. 
There also exist
knowledge distillation techniques that 
are utilized to integrate clients' knowledge into the global model~\cite{lin2020ensemble,zhu2021data}, 
which have been shown to be more effective than simple parameter averaging.





\subsection{Continual Learning}
There is an implicit connection between continual learning and federated learning. 
Continual learning is a concept of learning a series of tasks sequentially without forgetting previously learned knowledge, 
while
federated learning is a special multi-task learning paradigm, where all tasks (clients) are learned in parallel. 
These two learning paradigms could suffer from the same challenge that the tasks negatively affect each other when their data distributions vary considerably, which exacerbates {catastrophic forgetting}~\cite{french1999catastrophic} in continual learning and slows down the model convergence in federated learning. 

To solve the challenge, 
some continual learning methods have been proposed~\cite{li2017learning,kirkpatrick2017overcoming,lopez2017gradient}. 
For example, EWC~\cite{kirkpatrick2017overcoming} adds a penalty term to the objective function to penalize the modification on parameters that are most informative for prior tasks. 
GEM~\cite{lopez2017gradient} allows positive knowledge transfer to past tasks through gradient projection. 
Meanwhile,
these
approaches have been extended to federated learning
\cite{shoham2019overcoming,wang2021federated,esfandiari2021cross,li2021fedh2l}. 
For example, 
to alleviate local models drifting apart, FedCurv~\cite{shoham2019overcoming} adapts EWC~\cite{kirkpatrick2017overcoming} to federated learning, 
which uses the diagonal of the Fisher information matrix to protect the parameters that are important to all clients. 
Inspired by GEM~\cite{lopez2017gradient}, a decentralized learning method CGA~\cite{esfandiari2021cross} is developed, 
where the cross gradients (computed by neighbors) and local gradients are projected into an aggregated gradient. 
While our method utilizes the idea of continual learning,
we focus on 
parallel training clients instead of sequential tasks.








\subsection{Fairness} 
Another related research direction to our paper is fairness~\cite{li2019fair,mohri2019agnostic,wang2021federated}, 
which aims to derive a model that has comparable performances over all the clients.
For example,
AFL~\cite{mohri2019agnostic} formulates FL as a min-max optimization problem, where the global model can be optimized for any target distribution formed by a mixture of client distributions. 
Li~et~al.~\cite{li2019fair}
minimize a reweighted loss where the clients with worse performance are given higher relative weights. 
Further,
{FedFV}~\cite{wang2021federated} identifies the conflicting gradients with large differences in magnitudes as a cause of unfairness, and proposes to mitigate the issue via gradient projection. 
Instead of targeting at fairness,
our proposed method attempts to 
making all the clients benefit from federated training rather than achieve similar performance.


\section{Preliminaries}\label{sec:preliminaries}
We denote an FL system with $N$ clients 
as $U = \{u_1, \ldots, u_N\}$, each of which has a private dataset $D_i$ of size $\left|D_i\right|$. 
Generally, FL aims to search the optimal global model parameter $\theta^*$, 
which minimizes the 
average risk over all the clients 
while keeping clients' data private. 
The optimization objective can be formulated as:
\begin{equation}\label{preliminaries-def}
\small
    \min_{\theta \in \mathbb{R}^d} \left \{ F(\theta) := \frac{1}{N}\sum_{i=1}^{N} {F}_{i}( \theta) \right \}
\end{equation}
where $F_{i}(\theta) := \frac{1}{\left|D_i\right|} \sum_{(x, y) \in D_i} \mathcal{L}(\theta; (x, y))$ is the empirical loss of the global model parameterized by $\theta$ over the dataset of client $u_i$, and $\mathcal{L}$ is the cross-entropy loss widely used in classification tasks. 
Further,
our analysis on \fed\ is 
with the same setting as
FedSGD~\cite{mcmahan2017communication},
which uses full participation and full batch size,
and performs
only one local epoch update 
in each communication round.



\section{Methodology}\label{sec:methodology}
In this section, 
we formulate our optimization objective and present our framework \fed. 

\subsection{Consensus Mechanism} 
Instead of directly analyzing the global objective,
we propose a consensus mechanism to ensure risk decrease for each client
in each communication round.
Intuitively,
if the risk of each client decreases, 
the average risk over the whole population will also decrease. 
As the training proceeds, the global model 
could converge. 
With this perspective, 
we formulate our optimization objective in the $t_{th}$ round as: 
\begin{equation}
\label{optimization-objective}
    \begin{split}
        & \min_{\theta^{t+1}} \quad F(\theta^{t}) \\
        & \operatorname{s.t.} \quad F_k(\theta^{t+1}) \leq F_k(\theta^{t}), \forall u_k \in U.
    \end{split}
\end{equation}
In this way,
we minimize the global objective while 
constraining that each client benefits from federated training. 
Next, we present our consensus mechanism to solve the constrained optimization problem~(\ref{optimization-objective}). 
For ease of description, 
we first define \textit{consensus} 
on two clients. 

\begin{definition}[Consensus]
Given two clients $u_i$ and $u_j$,
if their gradients $g_i$ and $g_j$ satisfy
$g_i \cdot g_j \geq 0$,
we say that
client $u_i$ reaches a consensus with client $u_j$.
\end{definition}
Note that
the consensus mechanism 
is employed on the server side, 
where 
clients' gradients are allowed to be slightly modified
to promote consensus among clients. 
Meanwhile,
we have to 
restrain too much deviation of the modified client gradients 
from the original ones.
For notation simplicity, 
we omit the round index $t$
and formulate 
the objective of the consensus mechanism 
as: 
\begin{equation}\label{init_opt}
\begin{aligned}
    \begin{split}
    & \min_{\tilde{g}_i} \quad \frac{1}{2} \left \Vert \tilde{g}_i - g_i \right \Vert_2^2 \\
    & \operatorname{s.t.} \quad \tilde{g}_i \cdot g_j \geq 0, \forall u_i, u_j \in U.
    \end{split}
\end{aligned}
\end{equation}
where $\tilde{g}_i$ is the modified gradient of client $u_i$ and $g_j$ is the original gradient of client $u_j$. 
The consensus mechanism aims to optimize $\tilde{g}_i$,
based on which 
$\theta^t$
can be updated by $\theta^{t+1} \leftarrow \theta^t - \eta \frac{1}{N} \sum_{u_i \in U} \tilde{g}_i^t$. 
Here $\eta$ is the learning rate.
Our theoretical analysis 
shows that 
$\theta^{t+1}$
is 
the solution to problem~\ref{optimization-objective},
whose details
will be given in the next section.



To optimize problem ~(\ref{init_opt}), 
a naive approach is to directly take $\tilde{g}_i$ as parameters. 
However, $\tilde{g}_i$ is a $d$-dimensional vector,
where $d$ is the model size
and $d\gg N$.
This is very computationally expensive and poses significant challenge to training. 
Inspired by GEM~\cite{lopez2017gradient}, we represent $\tilde{g}_i$ as a linear combination of all the clients' gradients.
Specifically,
let $G = [g_1, g_2, \cdots, g_N] \in R^{d \times N}$ be the matrix of 
original gradients
and $A = [a_1, a_2, \cdots, a_N] \in R^{N \times N}$ be the coefficient matrix. 
For each $\tilde{g}_i$,
we have:
\begin{equation}
\label{eq:tildegi}
    \tilde{g}_i = G a_i,
\end{equation}
where
$a_i$ 
is the coefficient vector of client $u_i$.
Then
we substitute $\tilde{g}_i$ in Eq.~\ref{init_opt} with Eq.~\ref{eq:tildegi}
and derive
\begin{equation}\label{transformation}
    \begin{aligned}
        \begin{split}
            & \frac{1}{2} \Vert \tilde{g}_i - g_i \Vert_2^2 = \frac{1}{2} a_i^T G^T G a_i - g_i^T G a_i + \frac{1}{2} g_i^T g_i.
        \end{split}
    \end{aligned}
\end{equation}
We discard the constant term $\frac{1}{2} g_i^Tg_i$
and rewrite Eq.~\ref{init_opt} 
as:
\begin{equation}\label{qp_opt}
    \begin{aligned}
        \begin{split}
            & \min_{a_i} \quad \frac{1}{2} a_i^T G^T G a_i - g_i^T G a_i \\
            & \operatorname{s.t.} \quad -G^T G a_i \preceq 0, \forall u_i \in U.
        \end{split}
    \end{aligned}   
\end{equation}
In this way, the original optimization problem~(\ref{init_opt}) can be transformed into a quadratic programming (QP) problem. 
In each communication round, 
the server receives clients' gradients and solves the QP problem efficiently~\cite{gould1991algorithm,schubiger2020gpu}. 
After $a_i$ is computed,
we can calculate the modified gradient $\tilde{g}_i$,
which is further aggregated to update the parameters $\theta$ of the global model. 



\subsection{Theoretical Analysis}
\label{sec:proof}
We provide theoretical analysis in the setting of full participation, full batch size and only one local epoch update in each round,
which is the same as FedSGD.
It can
be easily extended to multiple epochs and minibatch settings when the learning rate is small enough. 
First, we make a
general assumption
as in~\cite{shen2019faster,liang2019variance}: 
\begin{restatable}{assumption}{asmLsmoothness}
\label{assumption:L-smoothness}
(Lipschitz Gradient)
Given a constant $L > 0$,
we assume that  
functions
$F_1, \ldots, F_N$ are differentiable and have $L$-Lipschitz gradients:
\begin{equation}
\label{lipschitz-gradient}
    \Vert \nabla F_i(x) - \nabla F_i(y) \Vert \leq L \Vert x - y \Vert, \forall x, y \in \mathbb{R}^d.
\end{equation}
\end{restatable}
\begin{restatable}{assumption}{asmLossBoundedBelow}
\label{loss bounded below}
    The sequence of iterates $\theta_t$ is contained in an open set over which F is bounded below by a scalar $F^{\star}$.
\end{restatable}

\begin{restatable}{assumption}{asmBoundedQPDrift}
\label{bounded QP drift}
(Bounded QP Drift)
we assume that  
$g_i^{t}$ and $\tilde{g}_i^t$ have bounded drifts.
There exists constant $\sigma_{QP} > 0$, such that
\begin{equation}
    \mathbb{E}||(\tilde{g}_i^t - g_i^{t})/(K\eta)  \|^2 \leq \sigma_{QP}^2
\end{equation}
where K is the number of local iterations, $\eta$ is the local learning rate.
\end{restatable}
\begin{restatable}{assumption}{asmBoundedLocalNoise}
\label{bounded local noise}
(Bounded Local Noise)
Given a constant $L > 0$,
we assume that  each client i has bounded local gradient noise.
There exists constant $\sigma_L > 0$, $\sigma_G > 0$ such that
\begin{align}
\begin{aligned}
\mathbb{E}_{i}\|\nabla_{\theta^{t}}F(\theta^{t}) - \nabla_{\theta^{t}}F_i(\theta^{t})  \|^2 \leq \sigma_{G}^2\\
\mathbb{E}_{i}\|\nabla_{\theta^{t}}F_i(\theta^{t}; b) - \nabla_{\theta^{t}}F_i(\theta^{t})  \|^2 \leq \sigma_{L}^2
\end{aligned}
\end{align}
\end{restatable}

Then base on Assumption~\ref{assumption:L-smoothness},
we have:
\begin{thm}\label{thm1}
With the consensus mechanism, SGD based FedCOME (see appendix algorithm \ref{alg:fedsgd}) can guarantee
$F(\theta^{t+1}) \leq F(\theta^{t})$. 
\end{thm} 

\begin{proof}
We have the following inequality for client $u_k \in \mathcal{U}$: 
\begin{equation}
\small
\label{lipschitz-gradient-k-client}
\begin{aligned}
\begin{split}
    F_k(\theta^{t+1}) & \le F_k(\theta^{t}) + \nabla F_k(\theta^{t}) \cdot (\theta^{t+1} - \theta^{t}) + \frac{1}{2}L\Vert \theta^{t+1} - \theta^{t} \Vert_2^2 \\ 
    & = F_k(\theta^{t}) - \eta \nabla F_k(\theta^{t}) \cdot (\frac{1}{N}\sum_{i = 1}^N \tilde{g}_i^t)  + \frac{1}{2} L \eta^2\Vert \frac{1}{N}\sum_{i = 1}^N \tilde{g}_i^t \Vert_2^2
\end{split}
\end{aligned}
\end{equation}

Since $\nabla F_k(\theta^t)=g_k^t$
and the last term is negligible if $\eta$ is small enough, i.e. $\eta \to 0^+$, we have
\begin{equation}\label{linear approximation}
    (F_k(\theta^{t+1}) - F_k(\theta^{t}))/\eta \le - \frac{1}{N} \sum_{i = 1}^N g_k^t \cdot \tilde{g}_i^t
\end{equation}

With the consensus mechanism, 
the modified gradients and the original gradients satisfy
$\tilde{g}_i \cdot g_k \geq 0, \forall u_i, u_k \in U$.
Therefore, 
$F_k(\theta^{t+1}) \leq F_k(\theta^{t})$ holds, implying that the average risk over the population keeps decreasing 
as the training proceeds. 
Further,
since $F(\theta^t) \geq 0$, $F(\theta^t)$ is guaranteed to converge. 
\end{proof}

\begin{restatable}{thm}{convergence}
Incorporating full participation from every worker in each round and considering the assumptions delineated in \ref{assumption:L-smoothness}, \ref{loss bounded below}, \ref{bounded QP drift}, and \ref{bounded local noise}, the sequence $\theta^t$ generated by the FedCOME algorithm \ref{alg:fedcome} exhibits the following property. 

Given that the local learning rate $\eta$ and global learning rate $\eta_g$ are selected to satisfy $\eta < \frac{1}{\sqrt{60}KL}$ and $\eta\eta_g \le \frac{1}{2KL}$, and provided that a constant $c$ lies within the interval $0 < c < \frac{1}{4} - 15K^2\eta^2L^2$, the inequality below holds true:

\begin{equation}
\frac{1}{T}\sum_{t = 0}^T \mathbb{E} ||\nabla F(\theta^t)||_2^2 \leq \frac{F(\theta^0) - F^{\star}}{c\eta K T} + \Phi
\end{equation}
where $\Phi = \frac{1}{c}[\frac{L \eta_g \eta}{N} \sigma_L^2 + \frac{5 K \eta^2 L^2}{2}(\sigma_L^2+6 K \sigma_G^2) + \sigma_{QP}^2 + \eta_g\eta KL\sigma_{QP}^2]$. 
\end{restatable}

\begin{corollary}
For $\eta = \mathcal{O}(\frac{1}{\sqrt{T}KL})$, $\eta_g = \sqrt{KN}$,  $\sigma_{QP} = 
\mathcal{O}(\frac{1}{\sqrt{T}})$, the convergence rate of \fed~ is $\min \frac{1}{T}\sum_{t = 0}^T \mathbb{E} ||\nabla F(\theta^t)||_2^2 = \mathcal{O}(\frac{1}{\sqrt{NKT}} + \frac{1}{T})$.
\end{corollary}


Due to the page limitation, we defer the detailed analysis of \fed~ to the appendix.

For time complexity, FedCOME requires solving a quadratic programming (QP) problem of size $M$,  the number of {selected clients} per round. 
Solving the QP problem incurs $\mathcal{O}(M^{3.5})$ time complexity (worst-case),which is acceptable for a server with enough computing power due to $M << N$, the population size.

\subsection{Client Sampling}
In the previous section,
we have shown that 
all the clients' risks decrease after each round in the full participation setting. 
However, in practice, 
it is usual that 
only a subset of clients are available 
for training in each round. 
The consensus mechanism can only guarantee decreased risks for clients used in training. 
To generalize the consensus mechanism 
to unavailable
clients, 
we design a novel sampling strategy, 
where the {most representative} clients for the global population are selected for training.
Specifically,
we maintain a similarity table $S\in R^{N\times N}$ on the server 
such that 
$S_{ij}$ represents the similarity between the $i$-th and $j$-th clients. And $S_{ij}$  is initialized to 0. 
Based on $S$,
we can select a subset of clients $\mathcal{P}_t$ of size $M$ in the $t$-th round, 
where the sum of similarities between any two clients in $\mathcal{P}_t$ is minimized.
Formally,
the objective is formulated as:
\begin{equation}
\label{sampling-objective}
\begin{aligned}
    \min_{\mathcal{P}_t} \left \{ H(\mathcal{P}_t) := \sum_{u_i,u_j \in \mathcal{P}_t} S_{ij} \right \},
\end{aligned}
\end{equation}
which is a combinatorial optimization problem
and can be solved
by the simulated annealing method. 
Initially, 
we construct $\mathcal{P}_t^0$ by randomly selecting a subset of clients from $U$.
After that,
in the ($r+1$)-th iteration, 
we construct 
$\mathcal{P}_t^{r+1}$ 
by randomly replacing a client in $\mathcal{P}_t^r$ with a random client from $U/\mathcal{P}_t^r$.
Note that 
$\mathcal{P}_t^{r+1}$ 
will be accepted with a probability $\min\{1, \exp(\frac{-\left(H(\mathcal{P}_t^{r+1})-H(\mathcal{P}_t^r)\right)}{\tau})\}$, 
where
$\tau$ is the temperature parameter and decays with iterations;
otherwise, 
we set $\mathcal{P}_t^{r+1} = \mathcal{P}_t^{r}$.
To enhance the diversity of $\mathcal{P}_{t}$ and prevent the model from easily trapping into the local minimum,
we further introduce the \emph{exploitation-exploration} strategy. 
When constructing $\mathcal{P}_{t}$,
$\mu M$ clients are derived by
the simulated annealing method while others 
are selected randomly.
Here, $\mu$ is a pre-defined percentage.


After $\mathcal{P}_{t}$ is derived,
for any two clients 
$u_i,u_j\in \mathcal{P}_{t}$,
we compute 
their gradient cosine similarity, 
given by: 
\begin{equation}
    Q_{ij} = \frac{\left \langle g_i, g_j\right \rangle}{\Vert g_i\Vert \Vert g_j\Vert}.
\end{equation} 
Then we 
update $S_{ij}$ 
by
\begin{equation}
\label{update-s}
    S_{ij} \leftarrow \alpha S_{ij} + (1-\alpha) Q_{ij},
\end{equation}
where $\alpha$ is a balancing coefficient. 
In this way,
we will select clients whose data distributions are heterogeneous.
For clients $u_i$ and $u_j$ that have IID data distributions, 
they are more likely to generate similar gradients and large $S_{ij}$ value.
Therefore, only one of them will be selected.
Note that
our client sampling strategy covers heterogeneous data of clients as much as possible.
When one client is selected, 
others that are not selected but with IID data distributions can also benefit from the consensus mechanism.
This further boosts the generalizablity of the consensus mechanism. 
Finally,
we combine client sampling and consensus mechanism, 
and summarize the pseudocodes of \fed\
in Algorithm~\ref{alg:fedcome} of the supplementary materials.
The time complexity of our proposed client sampling strategy is $\mathcal{O}(\tau*N)$, where $\tau$ is the stopping-iterations and $\tau=600$ in our experiments.

\subsection{Discussion}
We next summarize the major difference between \fed\ and other two competitors including 
SCAFFOLD 
and FedFV from the perspective of consensus. 


SCAFFOLD uses control variate to mitigate the discrepancy between clients' gradients, which can be considered as a soft correction for promoting the consensus across clients. 
In contrast, FedCOME imposes a hard correction on clients' model updates. 

FedFV and \fed\
achieve the consensus across clients
through iterative gradient projection and quadratic programming, respectively.
The former is more likely to lose the original gradients information, while the latter only allows slight modification on client gradients. 




Further, 
\fed\ does not 
introduce additional communication cost 
and allows clients to be stateless, which broadens its applicability 
in cross-device FL settings.

\section{Experiments}\label{sec:experiments}

In this section, we conduct extensive experiments to evaluate the effectiveness of \fed.
Due to the space limitation,
we provide details of datasets and baselines in the supplementary materials.

\subsection{Datasets and Baselines}
We use four FL benchmark datasets, including MNIST~\cite{lecun1998gradient}, FEMNIST~\cite{caldas2018leaf}, CIFAR-10~\cite{krizhevsky2009learning} and CIFAR-100~\cite{krizhevsky2009learning}.
Further,
we consider 6 other methods as baselines, 
including 
FedAvg~\cite{mcmahan2017communication}, FedProx~\cite{li2020federated}, SCAFFOLD~\cite{karimireddy2020scaffold}, FedDyn~\cite{acar2020federated}, FedFV~\cite{wang2021federated} and FedDC~\cite{gao2022feddc}.
In particular,
{FedProx}, {SCAFFOLD}, {FedDyn} and {FedDC} are specially designed for addressing data heterogeneity, while 
{FedFV} deals with fairness across clients. 
\subsection{Implementation details}
We implement \fed\ using PyTorch and train it by SGD optimizer.
In our experiments,
the balancing coefficient $\alpha$ in Eq.~\ref{update-s} and the percentage $\mu$ are set to 0.5 and 0.7, respectively. 
For all the methods, 
we set the local minibatch size to 50, weight decay to 1e-3 and learning rate to 0.05 with a learning rate decay of 0.998 per round. 
Further,
clients are set to perform local optimization for five epochs on CIFAR-100 and one epoch on other datasets. 
For all the baselines,
we perform grid search to fine-tune their hyper-parameters, whose 
search space is 
summarized in Table~\ref{tab:seach-space} of Appendix.
We implement FedFV and FedDC with the official codes released by their authors
and use the codes released by~\cite{gao2022feddc} for other baselines.
For each dataset, 
we adopt the same neural network model for all the methods.
Specifically, we adopt a fully connected network with one hidden layer of 64 units on MNIST.
Similar as in~\cite{mcmahan2017communication}, 
for CIFAR10, CIFAR-100 and FEMNIST, 
we use a ConvNet~\cite{lecun1998gradient} that contains 2 convolutional layers and 2 fully connected layers. 
All the experiments are run on one Tesla V100 GPU. 

\subsection{Results in Full Participation}
We next force all the clients to participate in each round of training
and show the classification results in table~\ref{tab:convergence-accuracy}.
We evaluate the performance of all the methods w.r.t. the weighted average accuracy, where the weights are proportional to the client sample sizes.
The accuracy is reported
with mean and standard deviation over three random runs.
From the table, 
we see that \fed~beats other competitors 
on MNIST, CIFAR-10, and CIFAR-100. 
While 
\fed\ is not the winner on
FEMNIST,
it can still achieve comparable performance with the best result. 
These observations show the effectiveness of 
our method \fed,
which optimizes the global objective while ensuring decreased risks for all the clients after each iteration. 

\begin{table}
\centering
\small
\resizebox{\linewidth}{!}{
\begin{tabular}{l c c c c }
\toprule
\specialrule{0em}{1.5pt}{1.5pt}
\midrule
\multirow{2}{*}{\bf Method}  
& \multirow{2}{*}{\bf MNIST}  
& \multirow{2}{*}{\bf FEMNIST}  
& \multirow{2}{*}{\bf CIFAR-10}  
& \multirow{2}{*}{\bf CIFAR-100}  
\\
& & & & 
\\
\hline
FedAvg
& 92.21 $\pm$ 0.18 
& 81.42 $\pm$ 0.11 
& 57.82 $\pm$ 0.24 
& 31.78 $\pm$ 0.89 
\\
FedProx
& 93.21 $\pm$ 0.01 
& 80.74 $\pm$ 0.07 
& 60.94 $\pm$ 0.10
& 31.39 $\pm$ 0.45 
\\
SCAFFOLD
& 92.55 $\pm$ 0.03 
& 79.32 $\pm$ 0.47 
& 66.77 $\pm$ 0.44 
& 28.11 $\pm$ 0.25 
\\
FedDyn
& 94.34 $\pm$ 0.51 
& \underline{86.00 $\pm$ 0.91} 
& \underline{74.44 $\pm$ 0.25} 
& \underline{35.92 $\pm$ 0.12} 
\\
FedFV
& 92.81 $\pm$ 0.24 
& 82.50 $\pm$ 0.06 
& 64.07 $\pm$ 0.27 
& 28.03 $\pm$ 0.82 
\\
FedDC
& \underline{95.22 $\pm$ 0.26} 
& \textbf{86.36 $\pm$ 0.34} 
& 72.64 $\pm$ 0.22 
& 22.47 $\pm$ 0.71 
\\
\midrule
FedCOME
& \textbf{95.61 $\pm$ 0.18} 
& 85.41 $\pm$ 0.05 
& \textbf{75.88 $\pm$ 0.36} 
& \textbf{37.66 $\pm$ 1.34} 
\\
\midrule
\specialrule{0em}{1.5pt}{1.5pt}
\bottomrule
\end{tabular}
}
\caption{
Performance comparison in the full participation setting. 
Best results are highlighted in bold, while the runner-ups' are underlined. 
}
\label{tab:convergence-accuracy}
\end{table}

To further study model efficiency,
we plot the convergence curves of all the methods on all the datasets in Figure~\ref{fig:convergence-curves}. 
From the figure,
we see that
\fed, FedDC and FedDyn are among the best three in terms of convergence speed.
While FedDC performs very well on MNIST and FEMNIST datasets,
it converges very slow on CIFAR-10 and CIFAR-100,
which shows its instability.
For \fed\ and FedDyn, 
they have comparable convergence speeds. 
For example,
both of them outperform other methods by a large margin
on CIFAR-10 and CIFAR-100.
However,
our method \fed\ can converge to higher accuracy values,
which are also clarified in Table~\ref{tab:convergence-accuracy}.
These results show that our method \fed\ is both effective and efficient.


\begin{figure*}
\centering
\subfigure[MNIST]{
\includegraphics[width=0.23\textwidth,trim=70 20 110 55,clip]{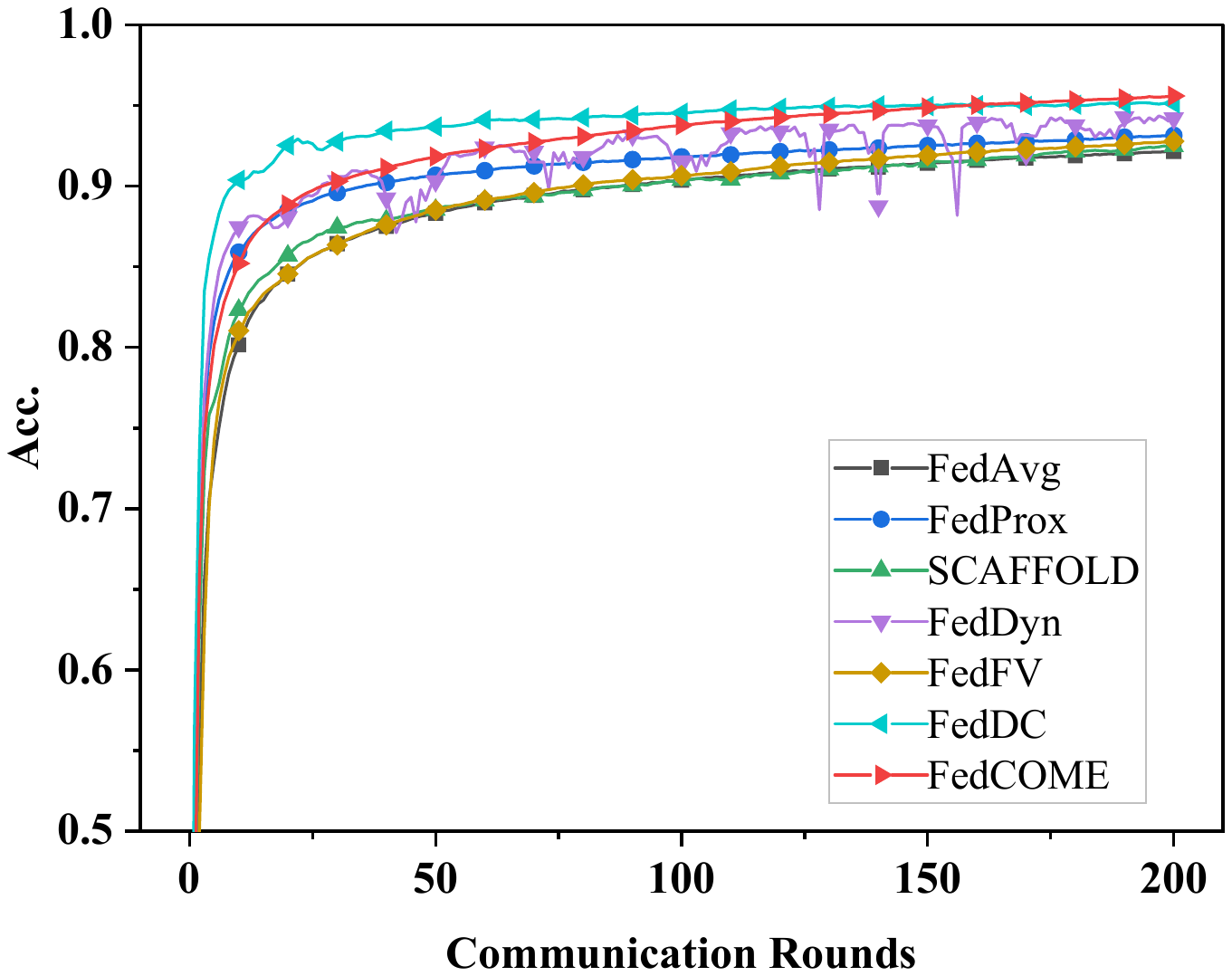}
}
\subfigure[FEMNIST]{
\includegraphics[width=0.23\textwidth,trim=70 20 110 55,clip]{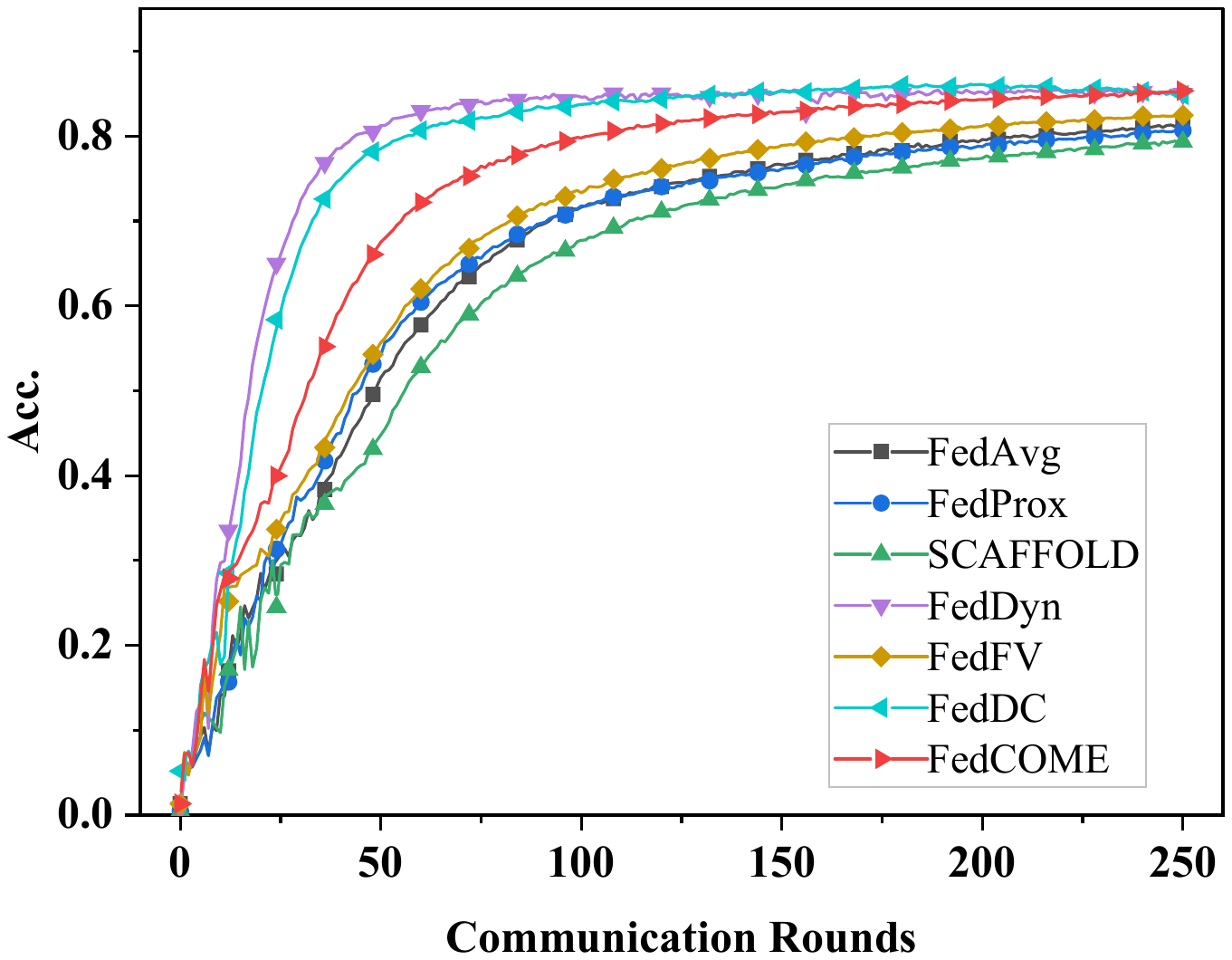}
}
\subfigure[CIFAR-10]{
\includegraphics[width=0.23\textwidth,trim=70 20 110 55,clip]{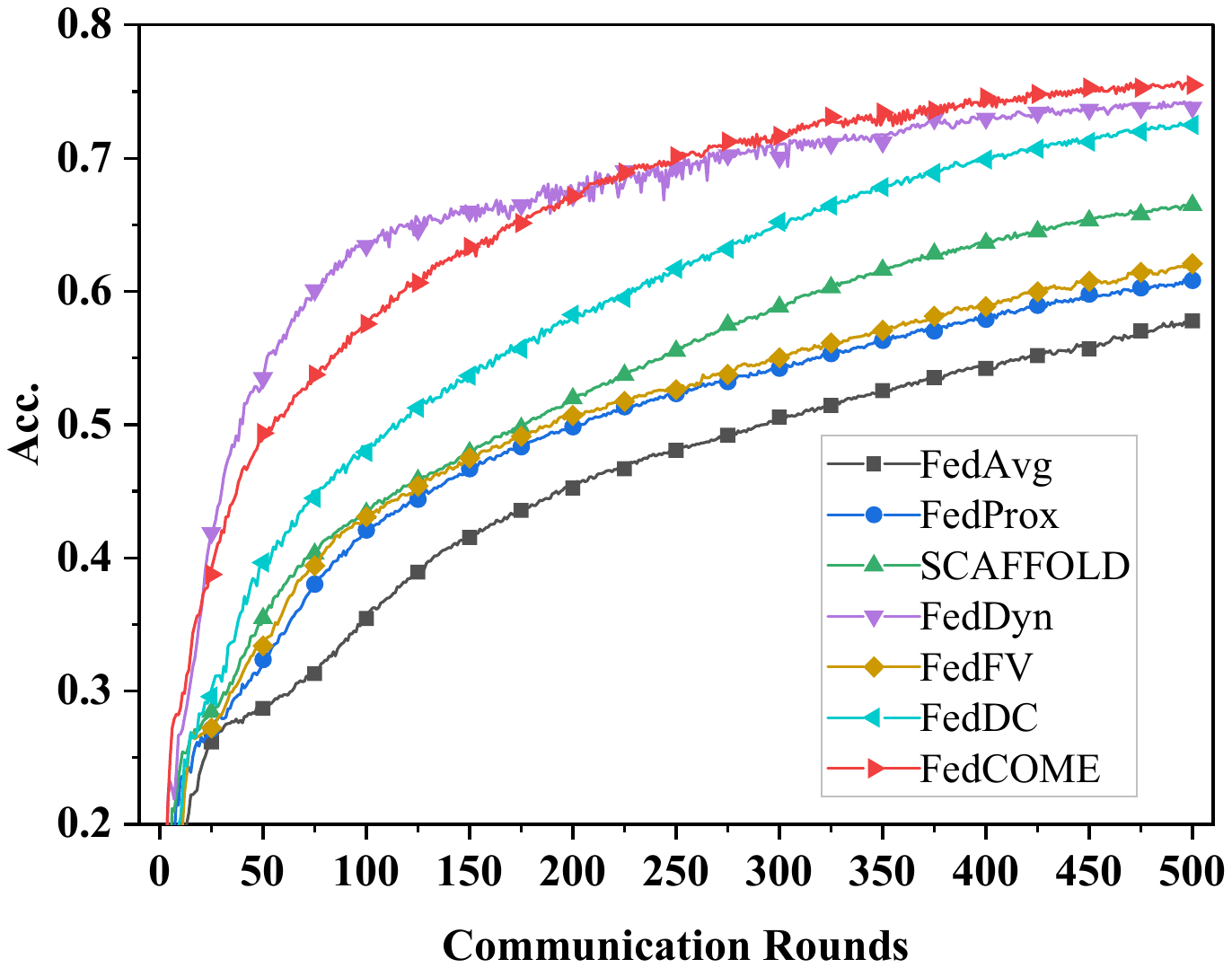}
}
\subfigure[CIFAR-100]{
\includegraphics[width=0.23\textwidth,trim=70 20 110 55,clip]{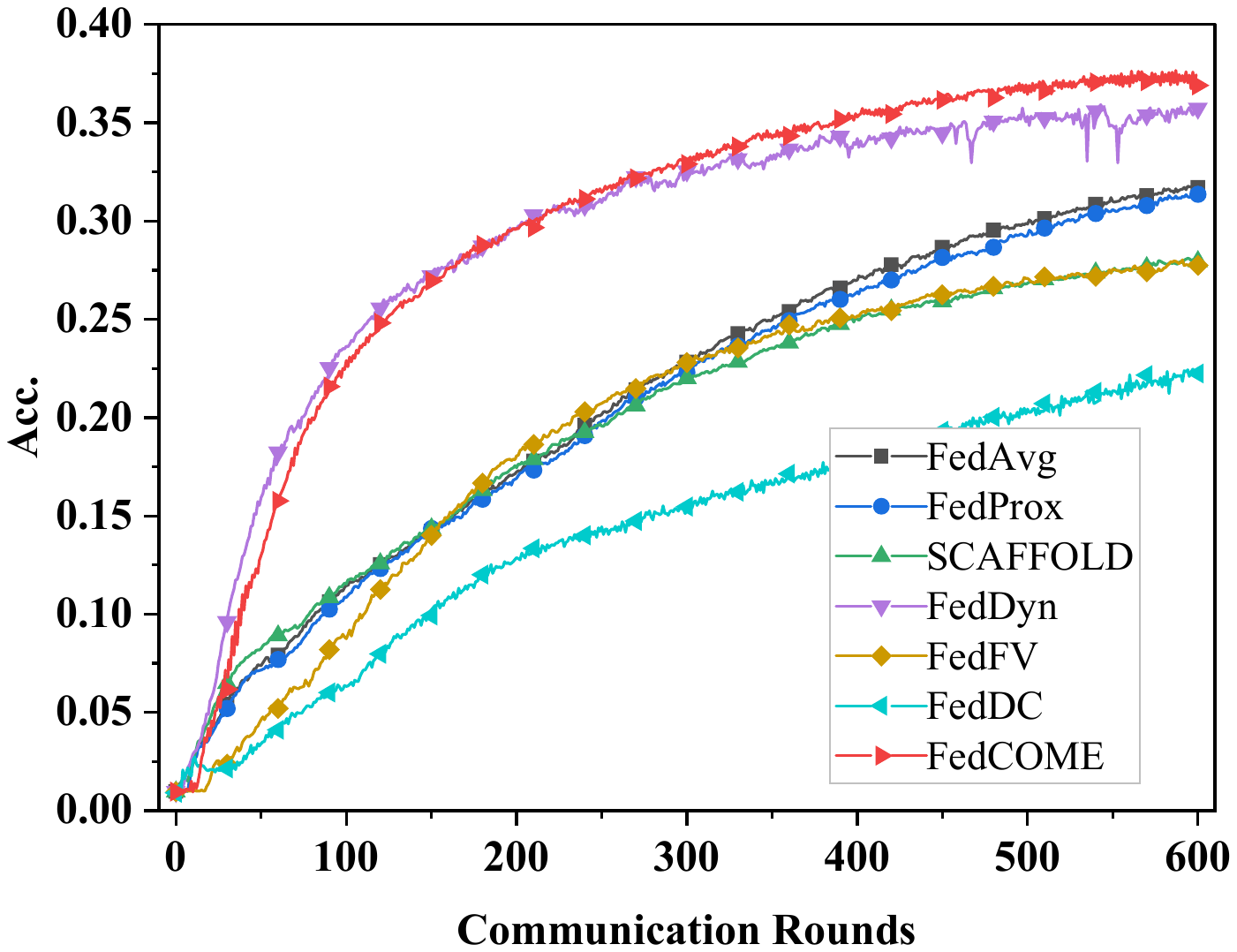}
}
\caption{Convergence comparison between \fed~and the baselines on all datasets.}
\label{fig:convergence-curves}
\end{figure*}


\subsection{Effect of Data Heterogeneity} 
To examine the robustness of \fed~against varying degrees of data heterogeneity, we further conduct experiments on CIFAR-10 to simulate different degrees of heterogeneity.
We set 
the number of classes $C \in \{2, 3, 4\}$ that each client has at most. 
The smaller $C$ is, the fewer data overlap between clients and the stronger data heterogeneity among clients. 
Experimental results are shown in Table~\ref{tab:data-heterogeneity}.
From the table,
we see that both \fed\ and FedDyn  
outperform other methods 
when meeting data heterogeneity.
Compared with FedDyn, 
\fed\ performs better as data heterogeneity increases.
This shows the advantage of \fed\ in dealing with data heterogeneity.
\begin{table}
    \centering 
    \resizebox{\linewidth}{!}{
    \begin{tabular}{@{}p{1.2cm} cccc@{}}
    \toprule
    \multirow{2}{*}{\bf Method} && \multicolumn{3}{c}{\textbf{\# Classes}}         \\
                                \cline{3-5} 
                                && \textbf{2}   &   \textbf{3}          & \textbf{4}        \\
    \midrule
    FedAvg                      && 57.82 $\pm$ 0.24 &   59.18 $\pm$ 0.41    & 63.97 $\pm$ 0.45  \\
    FedProx                     && 60.94 $\pm$ 0.10 &   62.72 $\pm$ 0.51    & 67.57 $\pm$ 0.24  \\
    SCAFFOLD                    && 66.77 $\pm$ 0.44 &   67.32 $\pm$ 0.46    & 67.72 $\pm$ 1.20  \\
    FedDyn                      && \underline{74.44 $\pm$ 0.25} &   \underline{76.21 $\pm$ 0.29} & \bf 80.20 $\pm$ 0.91  \\
    FedFV                       && 64.07 $\pm$ 0.27 &   65.70 $\pm$ 0.32  & 69.14 $\pm$ 0.83  \\
    FedDC                       && 72.64 $\pm$ 0.22 &   74.94 $\pm$ 1.71 & 75.31 $\pm$ 0.90  \\
    \midrule
    FedCOME                     && \bf 75.88 $\pm$ 0.36 &   \bf 76.82 $\pm$ 0.33  & \underline{78.12 $\pm$ 0.22} \\
    \bottomrule
    \end{tabular}
    }
    \caption{Effect of heterogeneity on CIFAR-10 dataset. Best results are in bold, while the runner-ups' are with underlines. }
    \label{tab:data-heterogeneity}
\end{table}

\begin{figure*}[t!]
    \centering
    \includegraphics[width=0.9\textwidth]{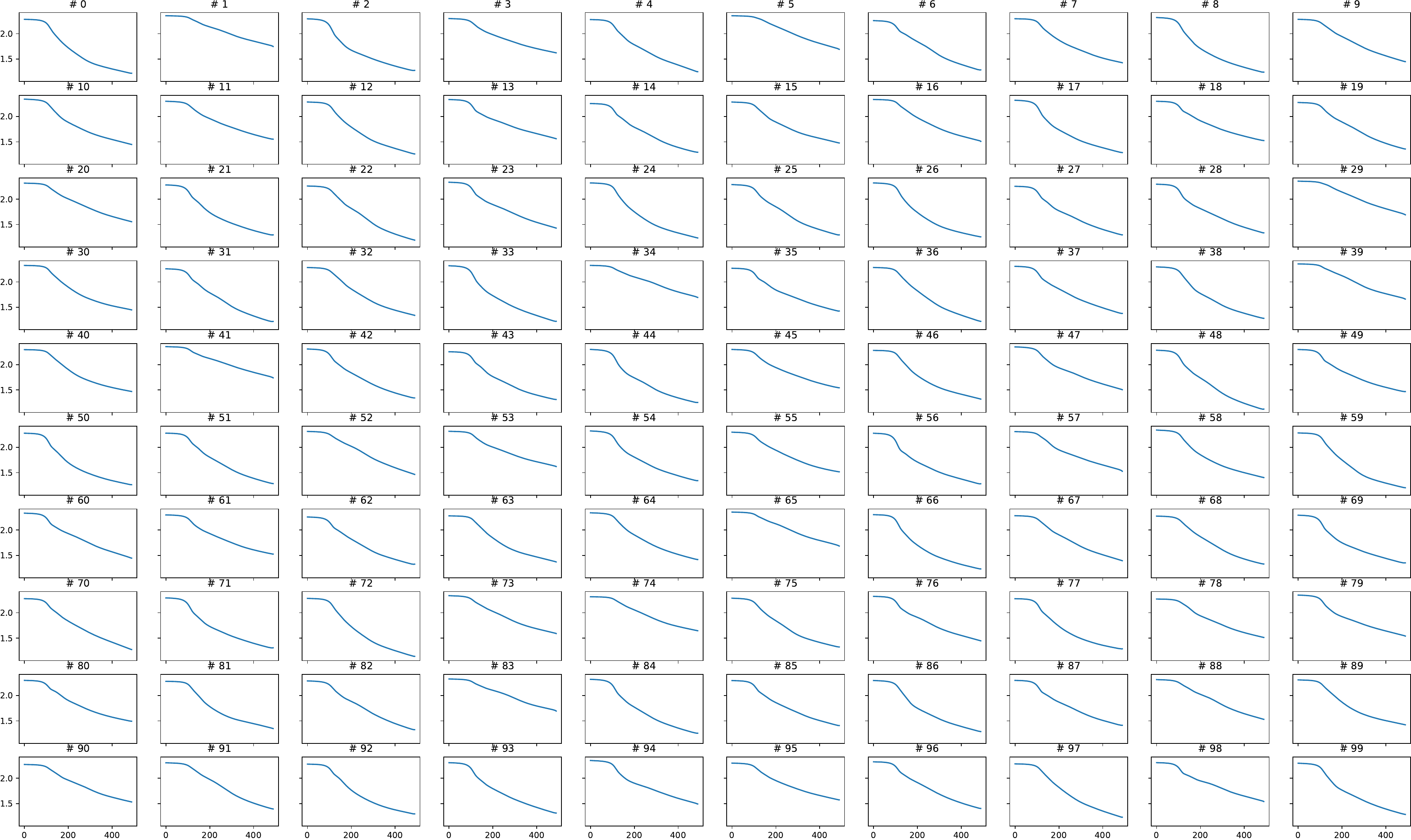}
    \caption{Loss curve of each client on CIFAR-10 dataset in full participation setting.}
    \label{fig:cifar10-full-decrease}
\end{figure*}

\subsection{Fairness}
\fed\ optimizes the global objective subject to decreased risks for all the clients after each communication round.
In other words,
all the clients can benefit from the federated training, which 
motivates further study on fairness.
Specifically,
we compare \fed~with FedFV, which is specially designed for handling fairness in FL. 
We plot the test accuracy distribution over the clients on FEMNIST and CIFAR-10 in Figure~\ref{fig:fairness-distribtion}. 
From the figure,
it is clear that 
\fed\ exhibits sharper 
distributions than FedFV on both datasets,
which shows that 
\fed\ can lead to better fairness among clients.

\begin{figure*}[t]
    \centering
    \includegraphics[width=0.95\textwidth]{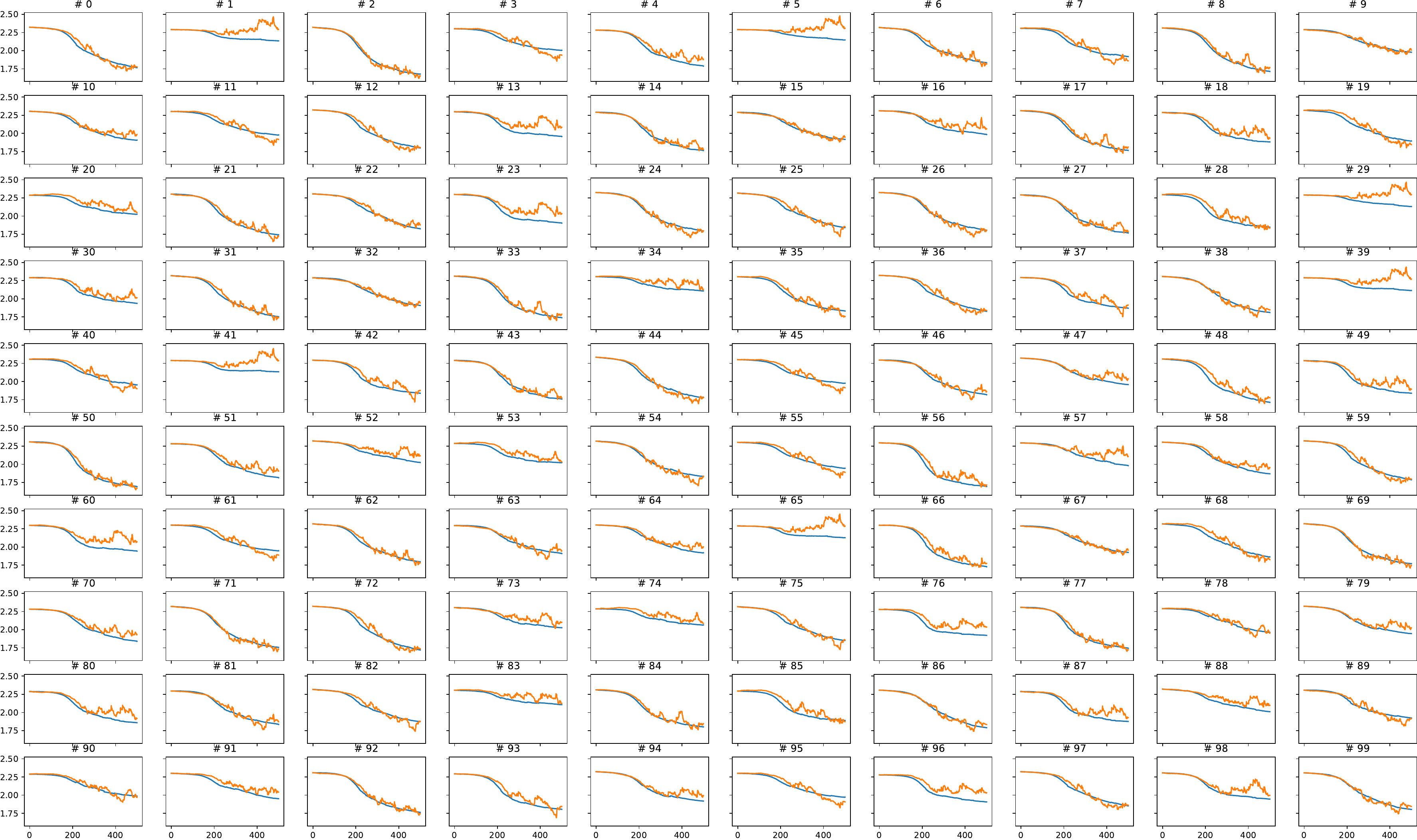}
    \caption{Loss curve of each client on CIFAR-10 dataset in the partial participation setting. \fed\ with our proposed client sampling strategy and random sampling are shown in blue and orange, respectively.}
    \label{fig:cifar10-partial-decrease}
\end{figure*}
\begin{figure*}[t!]
    \centering
    \includegraphics[width=1.0\linewidth,trim=0 435 0 0,clip]{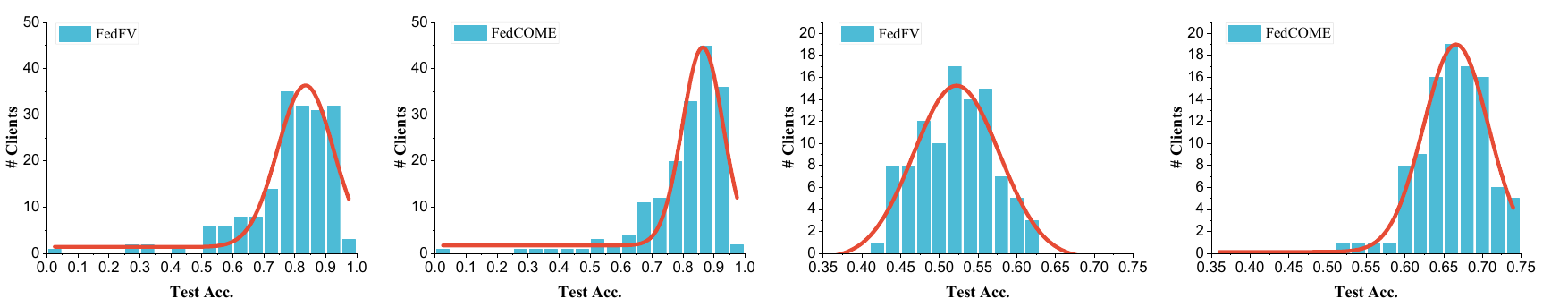}
    \caption{Distribution of test accuracy over all clients on FEMNIST (the first two panels) and CIFAR-10 (the last two panels). The x-axis represents the test accuracy interval, and the y-axis represents the number of clients whose accuracy falls within the corresponding interval.}
    \label{fig:fairness-distribtion}
\end{figure*}

\subsection{Results in Partial Participation}
To verify the effectiveness of our method in the partial participation FL scenario, 
we run all the methods on CIFAR-10 with the same settings as in the full participation scenario. 
We vary the participation ratio per round from $\{10\%, 20\%, 30\%, 40\%\}$. 
For each method, we run two variants: one with our proposed client sampling strategy (shown in regular fonts) and the other with random sampling (shown in small fonts). 
Experimental results are shown in Table~\ref{tab:partial-participation}. 
From the table, 
we observe that:
\begin{itemize}
    \item 
    \fed~beats other competitors across all the participation ratios with our proposed sampling strategy.
    With random sampling, \fed\ performs the best when the participation ratio is high (30\% and 40\%). 
    When the ratio is small (10\% and 20\%),
    FedDyn is the winner because 
    it can utilize
    the gradients of previously selected clients. 
    Despite this, \fed\ can achieve comparable performance with FedDyn. 
    These results show the superiority of \fed\ in handling data heterogeneity.
    \item All the methods with our proposed client sampling strategy generally perform better than using random sampling, which demonstrates the effectiveness of our client sampling strategy. 
    \item Our proposed sampling strategy can lead to large gain against random sampling when participation ratio is small.
    For example,
    when only 10\% of clients participate in training in each round, 
    \fed\ with 
    our client sampling strategy improves that with random sampling by 2.77\%.
    This shows that our sampling strategy can select more representative clients for the global data distribution.
\end{itemize}

\begin{table*}[htbp!]
    \centering
    \begin{tabular}{@{}l ccccc@{}}
    \toprule
        \multirow{3}{*}{\textbf{Method}}    &&  \multicolumn{4}{c}{\textbf{Participation Ratio}}  \\
    \cmidrule{3-6} 
    &&   \textbf{10\%} & \textbf{20\%} & \textbf{30\%}  & \textbf{40\%}\\
    \midrule
    FedAvg          && 60.31 $\pm$ 0.71 \ligry{(60.21 $\pm$ 0.30)} & 61.10 $\pm$ 0.83 \ligry{(60.67 $\pm$ 0.43)} & 61.29 $\pm$ 0.28 \ligry{(61.17 $\pm$ 0.24)}  & 62.34 $\pm$ 0.14 \ligry{(61.39 $\pm$ 0.51)}  \\
    FedProx         && 60.59 $\pm$ 0.35 \ligry{(60.47 $\pm$ 1.02)} & 61.70 $\pm$ 0.41 \ligry{(60.33 $\pm$ 0.29)} & 61.68 $\pm$ 0.19 \ligry{(60.92 $\pm$ 0.24)}  & 62.02 $\pm$ 0.59 \ligry{(61.42$\pm$ 0.17)} \\
    SCAFFOLD        && 57.65 $\pm$ 0.47 \ligry{(56.88 $\pm$ 0.34)} & 61.93 $\pm$ 0.41 \ligry{(60.23 $\pm$ 0.33)} & 62.11 $\pm$ 0.57 \ligry{(61.07 $\pm$ 0.37)}  & 62.06 $\pm$ 0.30 \ligry{(61.62 $\pm$ 0.43)} \\
    FedDyn          && \underline{67.73 $\pm$ 0.24} \ligry{(\bf{67.48 $\pm$ 0.24})} & \underline{70.67 $\pm$ 1.45} \ligry{(\textbf{69.68 $\pm$ 0.51})} & \underline{72.78 $\pm$ 0.21} \ligry{(\underline{71.84 $\pm$ 0.26})}  & \underline{73.56 $\pm$ 0.45} \ligry{(\underline{73.25 $\pm$ 0.60})} \\
    FedFV           && 61.00 $\pm$ 0.55 \ligry{(60.26 $\pm$ 0.42)} & 62.11 $\pm$ 0.43 \ligry{(61.74 $\pm$ 0.68)} & 63.34 $\pm$ 0.24 \ligry{(62.95 $\pm$ 0.18)}  & 63.81 $\pm$ 0.36 \ligry{(63.74 $\pm$ 0.16}) \\
    FedDC           && 59.09 $\pm$ 0.67 \ligry{(58.43 $\pm$ 1.30)} & 69.55 $\pm$ 1.23 \ligry{(69.16 $\pm$ 1.19)} & 70.99 $\pm$ 0.45 \ligry{(70.39 $\pm$ 1.30)}  & 71.59 $\pm$ 0.53 \ligry{(71.43 $\pm$ 0.60)} \\
    \midrule
    FedCOME         && \textbf{68.26 $\pm$ 0.74} \ligry{(\underline{65.49 $\pm$ 0.87})} & \textbf{72.37 $\pm$ 0.43} \underline{\ligry{(69.60 $\pm$ 0.59})} & \bf 73.81 $\pm$ 0.24 \ligry{(72.94 $\pm$ 0.41)}  & \bf 74.13 $\pm$ 0.16 \ligry{(73.98 $\pm$ 0.24)} \\
    \bottomrule
    \end{tabular}
    \caption{Comparison between FedCOME and other methods in the partial participation FL setting. Best results are marked in bold, while the runner-ups' are underlined. } 
    \label{tab:partial-participation}
\end{table*}






\subsection{Client Loss Study}
We end this section with an 
additional study on CIFAR-10 to show the loss of each client derived by \fed\ 
after each training round in both full and partial participation settings. 
We set learning rate to 0.01 
and each client performs local optimization for one epoch with full batch size. 
In the partial participation experiment, we set the participation ratio to 20\%. 
Figure~\ref{fig:cifar10-full-decrease} shows each client's loss curve in the full participation setting.
From the figure,
we see that each client's loss keeps decreasing with training epoches. 
This can be attributed to our consensus mechanism.
Further, in the partial participation setting, 
Fig.~\ref{fig:cifar10-partial-decrease} shows that
the loss curves of all the clients (including those clients not selected) derived by \fed\ 
with our proposed sampling strategy
keep decreasing while that with random sampling fluctuate frequently.
This further shows the effectiveness of our proposed sampling strategy.



\subsection{Sensitivity analysis of hyper-parameters}
We perform the hyper-parameter sensitivity study and show the results of FedCOME on CIFAR-10 dataset with a 20\% participation ratio. 
We present the experiment results in Table~\ref{table:sensitivity-analysis}. 
From the table,
we see that our method can achieve stable performance when
$\alpha \in \{0.5, 0.7, 0.9\}$ and $\mu \in \{0.5, 0.7, 0.9\}$. 
So we set $\alpha = 0.5$ and $\mu = 0.7$.

\begin{table}
\centering
\small
\resizebox{0.8\linewidth}{!}{
\begin{tabular}{c c c c c c}
\toprule
\diagbox{\textbf{$\alpha$}}{\textbf{$\mu$}} & \textbf{0.1}   & \textbf{0.3}   & \textbf{0.5}   & \textbf{0.7}   & \textbf{0.9} \\
\midrule
\textbf{0.1}   & 70.50 & 69.21 & 69.04 & 70.89 & 72.01 \\
\midrule
\textbf{0.3}   & 70.77 & 70.16 & 71.55 & 70.66 & 71.23 \\
\midrule
\textbf{0.5}   & 70.10 & 70.81 & 70.87 & 71.64 & 72.17 \\
\midrule
\textbf{0.7}   & 70.98 & 71.34 & 71.76 & 72.37 & 71.85 \\
\midrule
\textbf{0.9}   & 70.83 & 70.96 & 72.13 & 72.03 & 72.26 \\
\bottomrule
\end{tabular}
}
\caption{
Sensitivity analysis of hyper-parameters $\alpha$ and $\mu$.
}
\label{table:sensitivity-analysis}
\end{table}

\section{Conclusion}
In this paper, 
we studied data heterogeneity in FL and proposed \fed\ that provides a new perspective on the convergence of the global objective.
We presented a consensus mechanism to guarantee decreased risk for each client after each communication round.
Based on the consensus mechanism,
we theoretically proved the decrease of the global risk over the whole population.
To further generalize the consensus mechanism to partial participation FL scenario, we devised a novel client sampling strategy. 
We conducted extensive experiments on 4 benchmark datasets to show the effectiveness and efficiency of \fed. 
Also, we showed the advantage of \fed\ in ensuring fairness among clients. 
\label{sec:con}

\section*{Ethical Statement}
Same as FedAvg~\cite{mcmahan2017communication} and its variants like SCAFFOLD~\cite{karimireddy2020scaffold}, \fed\ also has a risk of information leakage due to the transmission of local gradients or model updates, but it can be easily mitigated with differential privacy or other defense techniques, such as gradient clipping~\cite{zhu2019deep}. 
\bibliography{ecai}

\begin{thebibliography}{10}

\bibitem{acar2020federated}
Durmus Alp~Emre Acar, Yue Zhao, Ramon Matas, Matthew Mattina, Paul Whatmough,
  and Venkatesh Saligrama, `Federated learning based on dynamic
  regularization', in {\em ICLR}, (2020).

\bibitem{caldas2018leaf}
Sebastian Caldas, Sai Meher~Karthik Duddu, Peter Wu, Tian Li, Jakub
  Kone{\v{c}}n{\`y}, H~Brendan McMahan, Virginia Smith, and Ameet Talwalkar,
  `Leaf: A benchmark for federated settings', {\em arXiv preprint
  arXiv:1812.01097}, (2018).

\bibitem{chen2021bridging}
Hong-You Chen and Wei-Lun Chao, `On bridging generic and personalized federated
  learning for image classification', in {\em International Conference on
  Learning Representations}, (2021).

\bibitem{chen2021fedmatch}
Jiangui Chen, Ruqing Zhang, Jiafeng Guo, Yixing Fan, and Xueqi Cheng,
  `Fedmatch: Federated learning over heterogeneous question answering data', in
  {\em Proceedings of the 30th ACM International Conference on Information \&
  Knowledge Management}, pp. 181--190, (2021).

\bibitem{Chen2020FedHealthAF}
Yiqiang Chen, Jindong Wang, Chaohui Yu, Wen Gao, and Xin Qin, `Fedhealth: A
  federated transfer learning framework for wearable healthcare', {\em IEEE
  Intelligent Systems}, {\bf 35}, (2020).

\bibitem{devlin2018bert}
Jacob Devlin, Ming-Wei Chang, Kenton Lee, and Kristina Toutanova, `Bert:
  Pre-training of deep bidirectional transformers for language understanding',
  {\em arXiv preprint arXiv:1810.04805}, (2018).

\bibitem{esfandiari2021cross}
Yasaman Esfandiari, Sin~Yong Tan, Zhanhong Jiang, Aditya Balu, Ethan Herron,
  Chinmay Hegde, and Soumik Sarkar, `Cross-gradient aggregation for
  decentralized learning from non-iid data', in {\em International Conference
  on Machine Learning}, pp. 3036--3046. PMLR, (2021).

\bibitem{french1999catastrophic}
Robert~M French, `Catastrophic forgetting in connectionist networks', {\em
  Trends in cognitive sciences}, {\bf 3}(4),  128--135, (1999).

\bibitem{gao2022feddc}
Liang Gao, Huazhu Fu, Li~Li, Yingwen Chen, Ming Xu, and Cheng-Zhong Xu, `Feddc:
  Federated learning with non-iid data via local drift decoupling and
  correction', in {\em Proceedings of the IEEE/CVF Conference on Computer
  Vision and Pattern Recognition}, pp. 10112--10121, (2022).

\bibitem{gatys2016image}
Leon~A Gatys, Alexander~S Ecker, and Matthias Bethge, `Image style transfer
  using convolutional neural networks', in {\em CVPR}, pp. 2414--2423, (2016).

\bibitem{gould1991algorithm}
Nicholas~IM Gould, `An algorithm for large-scale quadratic programming', {\em
  IMA Journal of Numerical Analysis}, {\bf 11}(3),  299--324, (1991).

\bibitem{hsu2019measuring}
Tzu-Ming~Harry Hsu, Hang Qi, and Matthew Brown, `Measuring the effects of
  non-identical data distribution for federated visual classification', {\em
  arXiv preprint arXiv:1909.06335}, (2019).

\bibitem{kairouz2021advances}
Peter Kairouz, H~Brendan McMahan, Brendan Avent, Aur{\'e}lien Bellet, Mehdi
  Bennis, Arjun~Nitin Bhagoji, Kallista Bonawitz, Zachary Charles, Graham
  Cormode, Rachel Cummings, et~al., `Advances and open problems in federated
  learning', {\em Foundations and Trends{\textregistered} in Machine Learning},
  {\bf 14}(1--2),  1--210, (2021).

\bibitem{karimireddy2020scaffold}
Sai~Praneeth Karimireddy, Satyen Kale, Mehryar Mohri, Sashank Reddi, Sebastian
  Stich, and Ananda~Theertha Suresh, `Scaffold: Stochastic controlled averaging
  for federated learning', in {\em ICML}.

\bibitem{kirkpatrick2017overcoming}
James Kirkpatrick, Razvan Pascanu, Neil Rabinowitz, Joel Veness, Guillaume
  Desjardins, Andrei~A Rusu, Kieran Milan, John Quan, Tiago Ramalho, Agnieszka
  Grabska-Barwinska, et~al., `Overcoming catastrophic forgetting in neural
  networks', {\em Proceedings of the national academy of sciences}, {\bf
  114}(13),  3521--3526, (2017).

\bibitem{krizhevsky2009learning}
Alex Krizhevsky, Geoffrey Hinton, et~al., `Learning multiple layers of features
  from tiny images', (2009).

\bibitem{lecun1998gradient}
Yann LeCun, L{\'e}on Bottou, Yoshua Bengio, and Patrick Haffner,
  `Gradient-based learning applied to document recognition', {\em Proceedings
  of the IEEE}, {\bf 86}(11),  2278--2324, (1998).

\bibitem{li2020federated}
Tian Li, Anit~Kumar Sahu, Manzil Zaheer, Maziar Sanjabi, Ameet Talwalkar, and
  Virginia Smith, `Federated optimization in heterogeneous networks', {\em
  Proceedings of Machine Learning and Systems}, {\bf 2},  429--450, (2020).

\bibitem{li2019feddane}
Tian Li, Anit~Kumar Sahu, Manzil Zaheer, Maziar Sanjabi, Ameet Talwalkar, and
  Virginia Smithy, `Feddane: A federated newton-type method', in {\em 2019 53rd
  Asilomar Conference on Signals, Systems, and Computers}, pp. 1227--1231.
  IEEE, (2019).

\bibitem{li2019fair}
Tian Li, Maziar Sanjabi, Ahmad Beirami, and Virginia Smith, `Fair resource
  allocation in federated learning', in {\em ICLR}, (2019).

\bibitem{li2019convergence}
Xiang Li, Kaixuan Huang, Wenhao Yang, Shusen Wang, and Zhihua Zhang, `On the
  convergence of fedavg on non-iid data', in {\em ICLR}, (2019).

\bibitem{li2021fedh2l}
Yiying Li, Wei Zhou, Huaimin Wang, Haibo Mi, and Timothy~M Hospedales, `Fedh2l:
  Federated learning with model and statistical heterogeneity', {\em arXiv
  preprint arXiv:2101.11296}, (2021).

\bibitem{li2017learning}
Zhizhong Li and Derek Hoiem, `Learning without forgetting', {\em IEEE
  transactions on pattern analysis and machine intelligence}, {\bf 40}(12),
  2935--2947, (2017).

\bibitem{liang2019variance}
Xianfeng Liang, Shuheng Shen, Jingchang Liu, Zhen Pan, Yifei Cheng, and Enhong
  Chen, `Variance reduced local sgd with lower communication complexity',
  (2019).

\bibitem{lin2020ensemble}
Tao Lin, Lingjing Kong, Sebastian~U Stich, and Martin Jaggi, `Ensemble
  distillation for robust model fusion in federated learning', {\em Advances in
  Neural Information Processing Systems}, {\bf 33},  2351--2363, (2020).

\bibitem{lin2020meta}
Yujie Lin, Pengjie Ren, Zhumin Chen, Zhaochun Ren, Dongxiao Yu, Jun Ma,
  Maarten~de Rijke, and Xiuzhen Cheng, `Meta matrix factorization for federated
  rating predictions', in {\em Proceedings of the 43rd International ACM SIGIR
  Conference on Research and Development in Information Retrieval}, pp.
  981--990, (2020).

\bibitem{lopez2017gradient}
David Lopez-Paz and Marc'Aurelio Ranzato, `Gradient episodic memory for
  continual learning', {\em Advances in neural information processing systems},
  {\bf 30}, (2017).

\bibitem{mcmahan2017communication}
Brendan McMahan, Eider Moore, Daniel Ramage, Seth Hampson, and Blaise~Aguera
  y~Arcas, `Communication-efficient learning of deep networks from
  decentralized data', in {\em Artificial intelligence and statistics}, pp.
  1273--1282. PMLR, (2017).

\bibitem{mohri2019agnostic}
Mehryar Mohri, Gary Sivek, and Ananda~Theertha Suresh, `Agnostic federated
  learning', in {\em International Conference on Machine Learning}, pp.
  4615--4625. PMLR, (2019).

\bibitem{muhammad2020fedfast}
Khalil Muhammad, Qinqin Wang, Diarmuid O'Reilly-Morgan, Elias Tragos, Barry
  Smyth, Neil Hurley, James Geraci, and Aonghus Lawlor, `Fedfast: Going beyond
  average for faster training of federated recommender systems', in {\em
  Proceedings of the 26th ACM SIGKDD International Conference on Knowledge
  Discovery \& Data Mining}, pp. 1234--1242, (2020).

\bibitem{pathak2020fedsplit}
Reese Pathak and Martin~J Wainwright, `Fedsplit: An algorithmic framework for
  fast federated optimization', {\em Advances in Neural Information Processing
  Systems}, {\bf 33},  7057--7066, (2020).

\bibitem{povey2011kaldi}
Daniel Povey, Arnab Ghoshal, Gilles Boulianne, Lukas Burget, Ondrej Glembek,
  Nagendra Goel, Mirko Hannemann, Petr Motlicek, Yanmin Qian, Petr Schwarz,
  et~al., `The kaldi speech recognition toolkit', in {\em IEEE 2011 workshop on
  automatic speech recognition and understanding}, number CONF. IEEE Signal
  Processing Society, (2011).

\bibitem{radford2019language}
Alec Radford, Jeffrey Wu, Rewon Child, David Luan, Dario Amodei, Ilya
  Sutskever, et~al., `Language models are unsupervised multitask learners',
  {\em OpenAI blog}, {\bf 1}(8), ~9, (2019).

\bibitem{reddy1976speech}
D~Raj Reddy, `Speech recognition by machine: A review', {\em Proceedings of the
  IEEE}, {\bf 64}(4),  501--531, (1976).

\bibitem{schubiger2020gpu}
Michel Schubiger, Goran Banjac, and John Lygeros, `Gpu acceleration of admm for
  large-scale quadratic programming', {\em Journal of Parallel and Distributed
  Computing}, {\bf 144},  55--67, (2020).

\bibitem{shen2019faster}
Shuheng Shen, Linli Xu, Jingchang Liu, Xianfeng Liang, and Yifei Cheng, `Faster
  distributed deep net training: Computation and communication decoupled
  stochastic gradient descent', {\em arXiv preprint arXiv:1906.12043}, (2019).

\bibitem{shoham2019overcoming}
Neta Shoham, Tomer Avidor, Aviv Keren, Nadav Israel, Daniel Benditkis, Liron
  Mor-Yosef, and Itai Zeitak, `Overcoming forgetting in federated learning on
  non-iid data', {\em arXiv preprint arXiv:1910.07796}, (2019).

\bibitem{wang2019federated}
Hongyi Wang, Mikhail Yurochkin, Yuekai Sun, Dimitris Papailiopoulos, and
  Yasaman Khazaeni, `Federated learning with matched averaging', in {\em
  International Conference on Learning Representations}, (2019).

\bibitem{wang2020tackling}
Jianyu Wang, Qinghua Liu, Hao Liang, Gauri Joshi, and H~Vincent Poor, `Tackling
  the objective inconsistency problem in heterogeneous federated optimization',
  {\em NeurIPS},  7611--7623, (2020).

\bibitem{wang2021federated}
Zheng Wang, Xiaoliang Fan, Jianzhong Qi, Chenglu Wen, Cheng Wang, and Rongshan
  Yu, `Federated learning with fair averaging', {\em arXiv preprint
  arXiv:2104.14937}, (2021).

\bibitem{zhao2003face}
Wenyi Zhao, Rama Chellappa, P~Jonathon Phillips, and Azriel Rosenfeld, `Face
  recognition: A literature survey', {\em ACM computing surveys (CSUR)}, {\bf
  35}(4),  399--458, (2003).

\bibitem{zhao2018federated}
Yue Zhao, Meng Li, Liangzhen Lai, Naveen Suda, Damon Civin, and Vikas Chandra,
  `Federated learning with non-iid data', {\em arXiv preprint
  arXiv:1806.00582}, (2018).

\bibitem{zhu2019deep}
Ligeng Zhu, Zhijian Liu, and Song Han, `Deep leakage from gradients', {\em
  Advances in neural information processing systems}, {\bf 32}, (2019).

\bibitem{zhu2021data}
Zhuangdi Zhu, Junyuan Hong, and Jiayu Zhou, `Data-free knowledge distillation
  for heterogeneous federated learning', in {\em International Conference on
  Machine Learning}, pp. 12878--12889. PMLR, (2021).

\end{thebibliography}

\newpage
\appendix
\setcounter{equation}{0}

\section{Datasets}
We use four FL benchmark datasets, including MNIST~\cite{lecun1998gradient}, FEMNIST~\cite{caldas2018leaf}, CIFAR-10~\cite{krizhevsky2009learning} and CIFAR-100~\cite{krizhevsky2009learning}.
Details on these datasets are summarized in Table~\ref{tab:stats}.

\begin{itemize}
    \item \textbf{MNIST} is a handwritten digit recognition dataset. Following~\cite{mcmahan2017communication}, we partition the dataset into 100 clients evenly, each of which has only 2 out of 10 classes.
    \item \textbf{FEMNIST} is a more complex version of MNIST. 
    It contains 62-class handwritten digits and letters. 
    This dataset is partitioned based on the writer of digit/letter, which is inherently heterogeneous. 
    \item \textbf{CIFAR-10} is a 10-class image classification dataset. 
    Similar as MNIST, 
    we construct a 
    federated dataset 
    composed of 100 clients, each of which has only two classes. 
    \item \textbf{CIFAR-100} is a 100-class image classification dataset. 
    We split it into 20 subsets by superclass labels, 
    where each subset is evenly distributed into 5 clients. 
    There are 100 clients in total and 
    data heterogeneity exists between clients from different subsets. 
\end{itemize}

\begin{table}
    \centering
    \begin{tabular}{p{1.4cm} c}
    \toprule
        \bf Method  & \bf Hyper-parameter space \\
        \midrule
        FedAvg      & - \\
        FedProx     & $\mu \in \{0.001, 0.01, 0.1\}$ \\
        SCAFFOLD    & - \\
        FedDyn      & $\alpha \in \{0.01,0.1, 1\}$ \\
        FedFV       & $\alpha \in \{0, 0.1, 0.2, \frac{1}{3}, 0.5, \frac{2}{3}\}, \tau \in\{0, 1, 3, 10\}$ \\
        FedDC       & $\alpha \in\left\{0.01,0.1, 1\right\}$ \\
    \bottomrule
    \end{tabular}
    \caption{The hyper-parameter space of the grid search for all baselines. "-" denotes the corresponding method has no tunable hyper-parameter.}
    \label{tab:seach-space}
\end{table}

\begin{table*}[h!]
\centering
\small
\begin{tabular}{ l c c c c c} 
\toprule
\specialrule{0em}{1.5pt}{1.5pt}
\midrule
\multirow{2}{*}{\bf Dataset} & \bf \multirow{2}{*}{\#(Training, testing)} & \bf \multirow{2}{*}{\bf \#Classes} & \multirow{2}{*}{\bf \#Clients} &  \bf \multirow{2}{*}{\#(Training, testing)/client} & \bf \multirow{2}{*}{\bf \#Classes/client}\\
\\
\midrule
MNIST       & (60000, 10000)    & 10    & 100   & (600, 100)    & 2 \\
FEMNIST     & (36692, 4178)     & 62    & 196   & (187, 21)     & - \\
CIFAR-10    & (50000, 10000)    & 10    & 100   & (500, 100)    & 2\\
CIFAR-100   & (50000, 10000)    & 100   & 100   & (500, 100)    & 5\\
\midrule
\specialrule{0em}{1.5pt}{1.5pt}
\bottomrule
\end{tabular}
\caption{
Overall statistics of datasets used for experiments. 
FEMNIST is heterogeneous by nature, while the others follow pathological non-IID partitions~\cite{mcmahan2017communication}. 
\#Classes/client: the number of classes a client most has in the pathological data partition setting, and "-" denotes a dataset is naturally heterogeneous. 
\#(Training, test)/client: average size of client's training and test sets. 
}
\label{tab:stats}
\end{table*}

\section{Baselines}
We consider 6 other methods as baselines, 
whose details are given as follows.
In particular,
{FedProx}, {SCAFFOLD}, {FedDyn} and {FedDC} are specially designed for addressing data heterogeneity, while 
{FedFV} deals with fairness across clients. 

\begin{itemize}
    \item FedAvg~\cite{mcmahan2017communication} is the vanilla FL method that 
    aggregates model updates weighted by their sample sizes.
    \item FedProx~\cite{li2020federated} adds a proximal item to each client's local optimization objective to prevent local model from diverging from global model.
    \item SCAFFOLD~\cite{karimireddy2020scaffold}
    regularizes the local training with estimated control variates to alleviate client-drift caused by heterogeneity.
    \item FedDyn~\cite{acar2020federated} updates the local regularizer dynamically to ensure the minima of the local empirical loss conforms with that of the global empirical risk. 
    \item FedFV~\cite{wang2021federated} performs gradient projection between clients' conflicting gradients to address the unfairness issue. 
    \item FedDC~\cite{gao2022feddc} utilizes an auxiliary local drift variable to track and bridge the gap between the parameters of the local and the global model.
\end{itemize}
\section{Pseudocodes}
We provide the pseudocodes of FedSGD and \fed\ in Algorithm~\ref{alg:fedsgd} and~\ref{alg:fedcome}, respectively.
\begin{algorithm}[h]
\caption{SGD based \fed}
\label{alg:fedsgd}
\textbf{Input}: initialized parameter $\theta^0$, total communication rounds $T$ and learning rate $\eta$ \\
\textbf{Output}: $\theta^T$
\begin{algorithmic}[1] 
\FOR{each round $t=0,1,...,T-1$}
    \STATE Server sends $\theta^t$ to all clients in $U$ \\
    \FOR{each client $i \in U$}
        \STATE Estimate the gradient of $F_i$ w.r.t. $\theta$ using batch $b_i$ \\
        $g_i^t = \nabla_{\theta} F_i(\theta^t; b_i)$ \\
    \ENDFOR
    \STATE Client $i$ sends $g_i^t$ back to server \\
    \FOR{$i \in \{1, 2, \cdots, M$\}}
        \STATE $\tilde{g}_i^t \leftarrow QP(\{g_j^t|j \in U \})$ \;
    \ENDFOR
    \STATE Server updates global model: $\theta^{t+1} = \theta^t - \eta \frac{1}{N} \sum_{i \in U} \tilde{g}_i^t$ \\
\ENDFOR
\end{algorithmic}
\end{algorithm}

\begin{algorithm}[t]
\caption{fedAvg based FedCOME}
\label{alg:fedcome}
\textbf{Input}: initialized parameter $\theta^0$,
total communication rounds $T$,local update epochs $E$,
local minibatch size $B$ \\
\textbf{Output}: $\theta^T$
\begin{algorithmic}[1] 
\FOR{each round $t = 1, \cdots, T$}
    \STATE Sample active client set $\mathcal{P}_t$ of size $M$ via optimizing problem~\ref{sampling-objective} \\ 
    \FOR{each client $i \in \mathcal{P}_t$}
        \STATE $\mathcal{B}_k \leftarrow$ Split local dataset $D_i$ into batches of size $B$ \\
        \STATE $\theta^{t}_0 \leftarrow \theta^{t-1}$ \\
           \FOR{each local batch  $b_{i,k}$ from 0 to $K-1$}
                \STATE $g_{i, k}^t = \nabla_{\theta^{t}_{i, k}}F_i(\theta^{t}_{i, k}; b_{i,k})$
    	        \STATE $\theta^{t}_{i, k+1} \leftarrow \theta^{t}_{i, k} - \eta g_{i, k}^t$
	    \ENDFOR
	    \STATE $g_i^{t} \leftarrow \theta^{t}_{i, K} - \theta^{t}_{i, 0}$ \\
    \ENDFOR
    
    \FOR{$i \in \{1, 2, \cdots, M$\}}
        \STATE $\tilde{g}_i^t \leftarrow QP(\{g_j^t|j \in P_t \})$ \;
    \ENDFOR
    \STATE $\theta^{t+1} \leftarrow \theta^t - \eta_g\frac{1}{M} \sum_{i \in \mathcal{P}_t} \tilde{g}_i^t$ \\
\ENDFOR
\end{algorithmic}
\end{algorithm}

\section{Convergence Analysis}
Here we provide convergence analysis for algorithm  \ref{alg:fedcome}.

\subsection{Assumption}

\asmLsmoothness*
\asmLossBoundedBelow*
\asmBoundedQPDrift*
\asmBoundedLocalNoise*

\subsection{Lemma}
\begin{lemma}
\label{g_QP to g}
Let all assumptions hold. Let $g_i$ be the update of client i  for all $i \in[N]:=\{1,2, \ldots, N\}$. Thus the following relationship holds
$$
\mathbb{E}[\|\frac{1}{N} \sum_{i=1}^N \tilde{\mathbf{g}}_i\|^2] \leq 2 K^2\sigma_{QP}^2 + 2\mathbb{E}\left[\left\|\frac{1}{N} \sum_{i=1}^N \mathbf{g}^i\right\|^2\right]  .
    $$
\end{lemma}
\begin{proof}
$$
\begin{aligned}
&\mathbb{E}\left[\left\|\frac{1}{N} \sum_{i=1}^N \tilde{\mathbf{g}}^i\right\|^2\right] \\
&=\mathbb{E}\left[\left\|\frac{1}{N} \sum_{i=1}^N\left(\tilde{\mathbf{g}}^i-\mathbf{g}^i+\mathbf{g}^i\right)\right\|^2\right]\\
&=\mathbb{E}\left[\left\|\frac{1}{N} \sum_{i=1}^N\left(\tilde{\mathbf{g}}^i-\mathbf{g}^i\right)+\frac{1}{N} \sum_{i=1}^N \mathbf{g}^i\right\|^2\right] \\
\end{aligned}
$$
$$
\begin{aligned}
& \stackrel{a}{\leq} 2 \mathbb{E}\left[\left\|\frac{1}{N} \sum_{i=1}^N\left(\tilde{\mathbf{g}}^i-\mathbf{g}^i\right)\right\|^2+\left\|\frac{1}{N} \sum_{i=1}^N \mathbf{g}^i\right\|^2\right] \\
& \stackrel{b}{\leq} 2 \frac{1}{N^2} \mathbb{E}\left[N \sum_{i=1}^N\left\|\tilde{\mathbf{g}}^i-\mathbf{g}^i\right\|^2\right]+2\mathbb{E}\left[\left\|\frac{1}{N} \sum_{i=1}^N \mathbf{g}^i\right\|^2\right] \\
& \stackrel{c}{\leq} 2 K^2\sigma_{QP}^2 + 2\mathbb{E}\left[\left\|\frac{1}{N} \sum_{i=1}^N \mathbf{g}^i\right\|^2\right] 
\end{aligned}
$$
(a) refers to the fact that the inequality $\|\mathbf{a}+\mathbf{b}\|^2 \leq 2\|\mathbf{a}\|^2+2\|\mathbf{b}\|^2$. (b) holds as $\left\|\sum_{i=1}^N \mathbf{a}_i\right\|^2 \leq N \sum_{i=1}^N\left\|\mathbf{a}_i\right\|^2$. (c) follows from Assumption \ref{bounded QP drift}.
\end{proof}

\begin{lemma} \label{one round progress}
(one round progress) For any step-size satisfying $\eta \leq \frac{1}{8 L K}$, we can have the following results:
$$
\frac{1}{N} \sum_{i=1}^N \mathbb{E}[\|\theta^t_{i, k} -\theta^t\|^2] \leq 5 K \eta_L^2(\sigma_L^2+6 K \sigma_G^2)+30 K^2 \eta_L^2\|\nabla f(\theta^t)\|^2
$$
\end{lemma}
\begin{proof}
    For any worker $i \in[N]$ and $k \in[K]$, we have:
$$
\begin{aligned}
& \mathbb{E}[\|\theta_{i, k}^t-\theta^t\|^2]=\mathbb{E}[\|\theta_{i, k-1}^t-\theta^t-\eta g_{t, k-1}^t\|^2] \\
& \leq \mathbb{E}[\|\theta_{i, k-1}^t-\theta^t-\eta(g_{i, k-1}^t-\nabla F_i(\theta_{i, k-1}^t)
\\&+\nabla F_i(\theta_{i, k-1}^t)-\nabla F_i(\theta^t)+\nabla F_i(\theta^t)-\nabla f(\theta^t)+\nabla f(\theta^t))\|^2] \\
& \leq(1+\frac{1}{2 K-1}) \mathbb{E}[\|\theta_{i, k-1}^t-\theta^t\|^2]+\mathbb{E}[\|\eta(g_{i, k-1}^t-\nabla F_i(\theta_{i, k-1}^t))\|^2] \\
& +6 K \mathbb{E}[\|\eta(\nabla F_i(\theta_{i, k-1}^t)-\nabla F_i(\theta^t))\|^2]+6 K \mathbb{E}[\| \eta(\nabla F_i(\theta^t)-\nabla f(\theta^t))) \|^2]+6 K\|\eta \nabla f(\theta^t)\|^2 \\
& \leq(1+\frac{1}{2 K-1}) \mathbb{E}[\|\theta_{i, k-1}^t-\theta^t\|^2]+\eta^2 \sigma_L^2+6 K \eta^2 L^2 \mathbb{E}[\|\theta_{i, k-1}^t-\theta^t\|^2]
\\&+6 K \eta^2 \sigma_G^2+6 K\|\eta \nabla f(\theta^t)\|^2 \\
& =(1+\frac{1}{2 K-1}+6 K \eta^2 L^2) \mathbb{E}[\|\theta_{i, k-1}^t-\theta^t\|^2]+\eta^2 \sigma_L^2
\\&+6 K \eta^2 \sigma_G^2+6 K\|\eta \nabla f(\theta^t)\|^2 \\
& \leq(1+\frac{1}{K-1}) \mathbb{E}[\|\theta_{i, k-1}^t-\theta^t\|^2]+\eta^2 \sigma_L^2+6 K \eta^2 \sigma_G^2+6 K\|\eta \nabla f(\theta^t)\|^2
\end{aligned}
$$
Unrolling the recursion, we get:
$$
\begin{aligned}
&\frac{1}{m} \sum_{i=1}^m \mathbb{E}[\|\theta_{i, k}^t-\theta^t\|^2] 
\\ &\leq \sum_{p=0}^{k-1}(1+\frac{1}{K-1})^p[\eta^2 \sigma_L^2+6 K\eta^2 \sigma_G^2+6 K \| \eta \nabla f(\theta^t)) \|^2] \\
\\ & \leq(K-1)[(1+\frac{1}{K-1})^K-1][\eta^2 \sigma_L^2+6 K \eta^2\sigma_G^2+6 K  \| \eta \nabla f(\theta^t)) \|^2] \\
& \leq 5 K \eta^2(\sigma_L^2+6 K \eta^2 \sigma_G^2)+30 K^2 \eta^2\|\nabla f(\theta^t)\|^2
\end{aligned}
$$
This completes the proof.
\end{proof}

\convergence*

\begin{proof}
For convenience, we define,
\begin{align*}
\tilde{\Delta}^t \triangleq -\frac{1}{N} \sum_{i = 1}^N \tilde{g}_i^t \\
\bar{\Delta}^t \triangleq -\frac{1}{N} \sum_{i = 1}^N g_i^t
\end{align*}
With Assumption \ref{assumption:L-smoothness},
\begin{align}
\begin{aligned}
\label{L-smooth_expansion}
&E[F(\theta^{t+1})] \leq F({\theta}^t)+\langle\nabla F({\theta}^t), \mathbb{E}[{\theta}^{t+1}-{\theta}^t]\rangle+\frac{L}{2} \mathbb{E}[\|{\theta}^{t+1}-{\theta}^t\|^2]\\
&= F({\theta}^t)+\langle\nabla F({\theta}^t), \mathbb{E}[\eta_g \tilde{\Delta}^t + \eta_g\eta K \nabla F({\theta}^t) - \eta_g\eta K \nabla F({\theta}^t)] \rangle\\
    &+\frac{\eta_g^2L}{2} \mathbb{E}[\|\tilde{\Delta}^t|^2]\\
&=F(\mathbf{\theta}^t)-\eta_g \eta K\|\nabla F(\mathbf{\theta}^t)\|^2 
\\ &+ \eta_g\underbrace{\langle\nabla F(\mathbf{\theta}^t), \mathbb{E}_t[\bar{\Delta}^t+ K \nabla F(\mathbf{\theta}^t)]\rangle}_{A_1}
\\&+\eta_g \underbrace{\langle\nabla F(\mathbf{\theta}^t), \mathbb{E}_t[\tilde{\Delta}^t - \bar{\Delta}^t]\rangle}_{A_2} +\frac{L}{2} \eta_g^2 \underbrace{\mathbb{E}_t[\|\tilde{\Delta}^t\|^2]}_{A_3}
\end{aligned}
\end{align}

For $A_1=\langle\nabla F(\theta^t), \mathbb{E}_t[\bar{\Delta}^t+\eta K \nabla F(\theta^t)]\rangle $,

\begin{align*}
    & A_1=\langle\nabla F(\theta^t), \mathbb{E}_t[-\frac{1}{N} \sum_{i=1}^N \sum_{k=0}^{K-1} \eta \mathbf{g}_{i, k}^t+\eta K \nabla F(\theta^t)]\rangle \\
& =\langle\nabla F(\theta^t), \mathbb{E}_t[-\frac{1}{N} \sum_{i=1}^N \sum_{k=0}^{K-1} \eta \nabla F_i(\theta_{i, k}^k)+\eta K \frac{1}{N} \sum_{i=1}^N \nabla F_i(\theta^t)]\rangle \\
& =\langle\sqrt{\eta K} \nabla F(\theta^t),-\frac{\sqrt{\eta}}{N \sqrt{K}} \mathbb{E}_t \sum_{i=1}^N \sum_{k=0}^{K-1}(\nabla F_i(\theta_{i, k}^k)-\nabla F_i(\theta^t))\rangle \\
& \stackrel{(a 1)}{=} \frac{\eta K}{2}\|\nabla F(\theta^t)\|^2+\frac{\eta}{2 K N^2} \mathbb{E}_t\|\sum_{i=1}^N \sum_{k=0}^{K-1} \nabla F_i(\theta_{i, k}^k)-\nabla F_i(\theta^t)\|^2
\\&-\frac{\eta}{2 K N^2} \mathbb{E}_t\|\sum_{i=1}^N \sum_{k=0}^{K-1} \nabla F_i(\theta_{i, k}^k)\|^2 \\
& \\
& \stackrel{(a 2)}{\leq} \frac{\eta K}{2}\|\nabla F(\theta^t)\|^2+\frac{\eta}{2 N} \sum_{i=1}^N \sum_{k=0}^{K-1} \mathbb{E}_t\|\nabla F_i(\theta_{i, k}^k)-\nabla F_i(\theta^t)\|^2
\\&-\frac{\eta}{2 K N^2} \mathbb{E}_t\|\sum_{i=1}^N \sum_{k=0}^{K-1} \nabla F_i(\theta_{i, k}^k)\|^2 \\
& \\
& \stackrel{(a 3)}{\leq} \frac{\eta K}{2}\|\nabla F(\theta^t)\|^2+\frac{\eta L^2}{2 N} \sum_{i=1}^N \sum_{k=0}^{K-1} \mathbb{E}_t\|\theta_{i, k}^k-\theta^t\|^2
\\&-\frac{\eta}{2 K N^2} \mathbb{E}_t\|\sum_{i=1}^N \sum_{k=0}^{K-1} \nabla F_i(\theta_{i, k}^k)\|^2 \\
    & \stackrel{(a 4)}{\leq} \eta K(\frac{1}{2}+15 K^2 \eta^2 L^2)\|\nabla F(\theta_t)\|^2+\frac{5 K^2 \eta^3 L^2}{2}(\sigma_L^2+6 K \sigma_G^2)
\\&-\frac{\eta}{2 K N^2} \mathbb{E}_t\|\sum_{i=1}^N \sum_{k=0}^{K-1} \nabla F_i(\theta_{i, k}^k)\|^2.
\end{align*}
\\
where 
$(a 1)$ follows from that $\langle\mathbf{x}, \mathbf{y}\rangle=\frac{1}{2}[\|\mathbf{x}\|^2+\|\mathbf{y}\|^2-\|\mathbf{x}-\mathbf{y}\|^2]$ for $\mathbf{x}=\sqrt{\eta K} \nabla F(\theta_t)$ and $\mathbf{y}=-\frac{\sqrt{\eta}}{N \sqrt{K}} \mathbb{E}_t \sum_{i=1}^N \sum_{k=0}^{K-1}(\nabla F_i(\theta_{t, k}^i)-\nabla F_i(\theta_t))$, $(a 2)$ is due to that $\mathbb{E}[\|x_1+\cdots+x_n\|^2] \leq$ $n \mathbb{E}[\|x_1\|^2+\cdots+\|x_n\|^2]$, $(a3)$ is due to Assumption \ref{assumption:L-smoothness} and $(a4)$ follows from Lemma \ref{one round progress}.

The term $A_2$ in (\ref{L-smooth_expansion}) can be bounded as:

For $A_2 = \langle\nabla F(\mathbf{\theta}^t), \mathbb{E}_t[\tilde{\Delta}^t - \bar{\Delta}^t]\rangle$
$$
\begin{aligned}
A_2 &= \langle\nabla F(\mathbf{\theta}^t), \mathbb{E}^t[\tilde{\Delta}^t - \bar{\Delta}^t]\rangle \\ 
&=\langle \sqrt{\frac{K\eta}{2}}\nabla F(\mathbf{\theta}^t), \sqrt{\frac{2}{K\eta}}\mathbb{E}^t[\tilde{\Delta}^t - \bar{\Delta}^t]\rangle \\ 
& \stackrel{(a 5)}{\le} \frac{K\eta}{4} \|\nabla F(\mathbf{\theta}^t)\|^2 + K\eta \sigma_{QP}^2
\end{aligned}
$$
where 
$(a 5)$ follows from that $\langle\mathbf{x}, \mathbf{y}\rangle \le \frac{1}{2}[\|\mathbf{x}\|^2+\|\mathbf{y}\|^2]$ for $\mathbf{x}=\sqrt{\frac{K}{2}}\nabla F(\mathbf{\theta}^t)$ and $\mathbf{y}=\sqrt{\frac{2}{K}}\mathbb{E}^t[\tilde{\Delta}^t - \bar{\Delta}^t]$ and
assumption \ref{bounded QP drift}.
For $A_3 = \mathbb{E}_t[\|\tilde{\Delta}^t\|^2]$,
using lemma \ref{g_QP to g},
$$
\begin{aligned}
    &\mathbb{E}_t[\|\tilde{\Delta}^t\|^2] = \mathbb{E}[\|\frac{1}{N} \sum_{i=1}^N \tilde{\mathbf{g}}_i\|^2] \\ &\leq 2 \eta^2 K^2\sigma_{QP}^2 + 2\eta^2\mathbb{E}\left[\left\|\frac{1}{N} \sum_{i=1}^N \mathbf{g}^i\right\|^2\right]
    \\ &= 2 \eta^2 K^2\sigma_{QP}^2 + 2\mathbb{E}_t[\|\bar{\Delta}^t\|^2].
\end{aligned}
$$
Then,  for $\mathbb{E}_t[\|\bar{\Delta}^t\|^2]$
$$
\begin{aligned}
& \mathbb{E}_t[\|\frac{1}{N} \sum_{i=1}^N \Delta_i^t\|^2]]\\
& = \frac{1}{N^2} \mathbb{E}_t[\|\sum_{i=1}^N \Delta_i^t\|^2] \\
& =\frac{\eta^2}{N^2} \mathbb{E}_t[\|\sum_{i=1}^N \sum_{k=0}^{K-1} \mathbf{g}_{i, k}^t\|^2]] \\
& \stackrel{(a 6)}{=} \frac{\eta^2}{N^2} \mathbb{E}_t[\|\sum_{i=1}^N \sum_{k=0}^{K-1}(\mathbf{g}_{i, k}^t-\nabla F_i(\theta_{i, k}^t))\|^2]
\\ &+\frac{\eta^2}{N^2} \mathbb{E}_t\|\sum_{i=1}^N \sum_{k=0}^{K-1} \nabla F_i(\theta_{i, k}^t)\|^2 \\
& \stackrel{(a 7)}{\leq} \frac{K \eta^2}{N} \sigma_L^2+\frac{\eta^2}{N^2} \mathbb{E}_t\|\sum_{i=1}^N \sum_{k=0}^{K-1} \nabla F_i(\theta_{i, k}^t)\|^2
\end{aligned}
$$
where $(a6)$ follows from the fact that $\left.\mathbb{E}\left[\|\mathbf{x}\|^2\right]=\mathbb{E}\left[\|\mathbf{x}-\mathbb{E}[\mathbf{x}]\|^2\right]+\|\mathbb{E}[\mathbf{x}]\|^2\right]$ and $(a7)$ is due to the bounded variance assumption in Assumption \ref{bounded local noise} and the fact that $\mathbf{g}_{i, k}^t-\nabla F_i(\theta_{i, k}^t)$ are independent with zero mean.

That's to say 
$$
A_3 \le  2 \eta^2 K^2\sigma_{QP}^2 + \frac{2K \eta^2}{N} \sigma_L^2+\frac{2\eta^2}{N^2} \mathbb{E}_t\|\sum_{i=1}^N \sum_{k=0}^{K-1} \nabla F_i(\theta_{i, k}^t)\|^2
$$.

Substituting the inequalities $A_1$, $A_2$, $A_3$ into  inequality (\ref{L-smooth_expansion}). We have:

$$
\begin{aligned}
& \mathbb{E} [F(\theta^{t+1})]-F(\theta^{t}) \\
& =-\eta_g \eta K\|\nabla F(\mathbf{\theta}^t)\|^2 
\\ &+ \eta_g\underbrace{\langle\nabla F(\mathbf{\theta}^t), \mathbb{E}_t[\bar{\Delta}^t+ K \nabla F(\mathbf{\theta}^t)]\rangle}_{A_1}
\\&+\eta_g \underbrace{\langle\nabla F(\mathbf{\theta}^t), \mathbb{E}_t[\tilde{\Delta}^t - \bar{\Delta}^t]\rangle}_{A_2} +\frac{L}{2} \eta_g^2 
\underbrace{\mathbb{E}_t[\|\tilde{\Delta}^t\|^2]}_{A_3}\\
&\le F(\theta_t)-\eta_g \eta K(\frac{1}{4}-15 K^2 \eta^2 L^2)\|\nabla F(\theta_t)\|^2
\\ &+\frac{L K \eta_g^2 \eta^2}{N} \sigma_L^2 + \frac{5 K^2\eta_g \eta^3 L^2}{2}(\sigma_L^2+6 K \sigma_G^2)
\\ &- (\frac{\eta_g\eta}{2 K N^2} -\frac{\eta_g^2\eta^2L}{N^2}) \mathbb{E}_t\|\sum_{i=1}^N \sum_{k=0}^{K-1} \nabla F_i(\theta_{i, k}^t)\|^2\\
&+ K\eta_g\eta \sigma_{QP}^2 + \eta_g^2\eta^2 K^2L\sigma_{QP}^2 \\
&\stackrel{(a8)}{\le} F(\theta_t)-\eta_g \eta K(\frac{1}{4}-15 K^2 \eta^2 L^2)\|\nabla F(\theta_t)\|^2\\ 
&+\frac{L K \eta_g^2 \eta^2}{N} \sigma_L^2 + \frac{5 K^2\eta_g \eta^3 L^2}{2}(\sigma_L^2+6 K \sigma_G^2)\\
&+ K\eta_g\eta \sigma_{QP}^2 + \eta_g^2\eta^2 K^2L\sigma_{QP}^2 \\
&\stackrel{(a9)}{\le} F(\theta_t)-c \eta_g \eta \|\nabla F(\theta_t)\|^2\\ 
&+\frac{L K \eta_g^2 \eta^2}{N} \sigma_L^2 + \frac{5 K^2\eta_g \eta^3 L^2}{2}(\sigma_L^2+6 K \sigma_G^2)\\
&+ K\eta_g\eta \sigma_{QP}^2 + \eta_g^2\eta^2 K^2L\sigma_{QP}^2 \\
\end{aligned}
$$
where $(a8)$ follows from $\frac{\eta_g\eta}{2 K N^2} -\frac{\eta_g^2\eta^2L}{N^2} > 0$ if $\eta_g\eta > \frac{1}{2KL}$. $\frac{1}{4}-15 K^2 \eta^2 L^2 > c > 0$ if $\eta \le \frac{1}{\sqrt{60}KL}$.
\\
Rearranging and summing from $t=0, \cdots, T-1$, we have:
$$
\begin{aligned}
    &c \eta_g \eta K\|\nabla f(\theta^t)\|^2 \le f(\theta^0)  - f(\theta^t)
    \\&+\frac{L K \eta_g^2 \eta^2}{N} \sigma_L^2 + \frac{5 K^2\eta_g \eta^3 L^2}{2}(\sigma_L^2+6 K \sigma_G^2) + K\eta_g\eta \sigma_{QP}^2 + \eta_g^2\eta^2 K^2L\sigma_{QP}^2 \\
\end{aligned}
$$
which implies,
\begin{equation}
\frac{1}{T}\sum_{t = 0}^T \mathbb{E} ||\nabla F(\theta_t)||_2^2 \leq \frac{F(\theta_0) - F^{\star}}{c\eta K T} + \Phi
\end{equation}
where $\Phi = \frac{1}{c}[\frac{L \eta_g \eta}{N} \sigma_L^2 + \frac{5 K \eta^2 L^2}{2}(\sigma_L^2+6 K \sigma_G^2) + \sigma_{QP}^2 + \eta_g\eta KL\sigma_{QP}^2]$. 
\\
This concludes the proof.

\end{proof}


\end{document}